\documentclass{article}

     \PassOptionsToPackage{numbers, compress}{natbib}
\usepackage[final]{neurips_2019}

\usepackage[utf8]{inputenc} % allow utf-8 input
\usepackage[T1]{fontenc}    % use 8-bit T1 fonts
\usepackage{hyperref}       % hyperlinks
\usepackage{url}            % simple URL typesetting
\usepackage{booktabs}       % professional-quality tables
\usepackage{amsfonts}       % blackboard math symbols
\usepackage{nicefrac}       % compact symbols for 1/2, etc.
\usepackage{microtype}      % microtypography
\usepackage{wrapfig}
\usepackage{amssymb,stmaryrd}
\usepackage{amsmath,amssymb,amsthm,amscd,epsfig,amsfonts,rotating,bm}
\usepackage{graphicx}
\usepackage{epsfig}
\usepackage{multirow}
\usepackage[ruled, vlined]{algorithm2e}
\usepackage[table]{xcolor}
\usepackage{enumitem}
\usepackage[font=small,labelfont=bf]{caption}

\makeatletter
\newtheorem*{rep@theorem}{\rep@title}
\newcommand{\newreptheorem}[2]{%
\newenvironment{rep#1}[1]{%
 \def\rep@title{#2 \ref{##1}}%
 \begin{rep@theorem}}%
 {\end{rep@theorem}}}
\makeatother

\newtheorem{defn}{Definition}
\newreptheorem{defn}{Definition}
\newtheorem{thm}{Theorem}
\newreptheorem{thm}{Theorem}

\newreptheorem{cor}{Corollary}

\DeclareMathOperator*{\argmax}{arg \, max}

\DeclareMathOperator*{\expe}{\mathbb{E}}

\definecolor{MyDarkGreen}{rgb}{0.02,0.6,0.02}

\title{A Generalized Algorithm for Multi-Objective \\ Reinforcement Learning and Policy Adaptation}

\author{%
  Runzhe Yang\\
  Department of Computer Science \\
  Princeton University \\
  \texttt{runzhey@cs.princeton.edu}\\
  \And
  Xingyuan Sun\\
  Department of Computer Science \\
  Princeton University \\
  \texttt{xs5@cs.princeton.edu}\\
  \And
  Karthik Narasimhan \\
  Department of Computer Science \\
  Princeton University \\
  \texttt{karthikn@cs.princeton.edu}\\
}

\begin{document}

\maketitle

\begin{abstract}
We introduce a new algorithm for multi-objective reinforcement learning (MORL) with linear preferences, with the goal of enabling few-shot adaptation to new tasks. In MORL, the aim is to learn policies over multiple competing objectives whose relative importance (\emph{preferences}) is unknown to the agent. While this alleviates dependence on scalar reward design, the expected return of a policy can change significantly with varying preferences, making it challenging to learn a single model to produce optimal policies under different preference conditions. We propose a generalized version of the Bellman equation to learn a single parametric representation for optimal policies over the space of all possible preferences. After an initial learning phase, our agent can execute the optimal policy under any given preference, or automatically infer an underlying preference with very few samples. Experiments across four different domains demonstrate the effectiveness of our approach.\footnote{Code is available at \url{https://github.com/RunzheYang/MORL}}
\end{abstract}

\section{Introduction}
\label{s:introduction}
\begin{wrapfigure}{R}{0.5\textwidth}
    \centering
    \vspace{-1em}
    \includegraphics[width=0.48\textwidth]{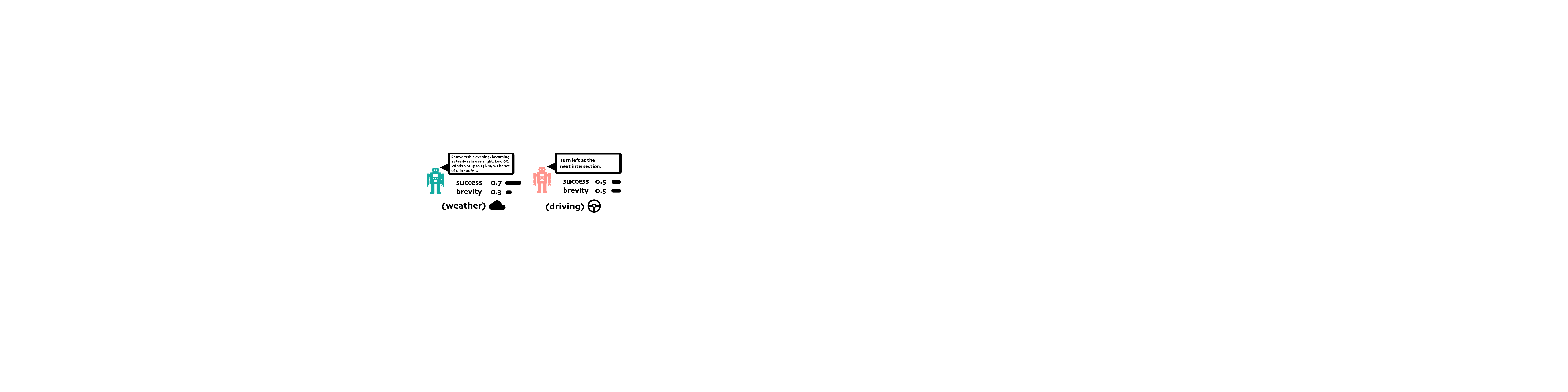}
    \vspace{-0.5em}
    \caption{\small Task-oriented dialogue policy learning is a real-life example of unknown linear preference scenario. Users may expect either briefer dialogue or more informative dialogue depending on the task.}
    \vspace{-0.5em}
    \label{fig:scenario}
%    \vspace{-10pt}
% \end{figure}
\end{wrapfigure}

In recent years, there has been increased interest in the paradigm of multi-objective reinforcement learning (MORL), which deals with learning control policies to simultaneously optimize over several criteria. Compared to traditional RL, where the aim is to optimize for a scalar reward, the optimal policy in a multi-objective setting depends on the \textit{relative preferences} among competing criteria. For example, consider a virtual assistant (Figure~\ref{fig:scenario}) that can communicate with a human to perform a specific task (e.g., provide weather or navigation information). Depending on the user's relative preferences between aspects like success rate or brevity, the agent might need to follow completely different strategies. If success is all that matters (e.g., providing an accurate weather report), the agent might provide detailed responses or ask several follow-up questions. On the other hand, if brevity is crucial (e.g., while providing turn-by-turn guidance), the agent needs to find the shortest way to complete the task. In traditional RL, this is often a fixed choice made by the designer and incorporated into the scalar reward. While this suffices in cases where we know the preferences of a task beforehand, the learned policy is limited in its applicability to scenarios with different preferences. The MORL framework provides two distinct advantages -- (1) reduced dependence on scalar reward design to combine different objectives, 
% \karthik{I changed this a bit - read it and see how you feel} 
% \runzhe{this is great!}
which is both a tedious manual task and can lead to unintended consequences~\cite{amodei2016concrete},  and (2) dynamic adaptation or transfer to related tasks with different preferences. 

% \begin{figure}[t]

% prior work
However, learning policies over multiple preferences under the MORL setting has proven to be quite challenging, with most prior work using one of two strategies~\cite{RoijersVWD13}. The first is to convert the multi-objective problem into a single-objective one through various techniques~\cite{kim2006adaptive,konak2006multi,nakayama2009sequential,lin2005min} and use traditional RL algorithms. These methods only learn an `average' policy over the space of preferences and cannot be tailored to be optimal for specific preferences. The second strategy is to compute a set of optimal policies that encompass the entire space of possible preferences in the domain~\cite{NatarajanT05,BarrettN08,MossalamARW16}. The main drawback of these approaches is their lack of scalability -- the challenge of representing a Pareto front (or its convex approximation) of optimal policies is handled by learning several individual policies, which can grow significantly with the size of the domain. 

% preferences as input, challenges with this.
% We tackle two concrete challenges in MORL by (1) providing theoretical convergence results of multi-objective versions of Q-Learning, and (2) demonstrating effective use of deep neural networks to tackle the feature representation problem.
In this paper, we propose a novel algorithm for learning a \textit{single policy network} that is optimized over the entire space of preferences in a domain. This allows our trained model to produce the optimal policy for any user-specified preference. We tackle two concrete challenges in MORL: (1) provide theoretical convergence results of a multi-objective version of Q-Learning for MORL with linear preferences, and (2) demonstrate effective use of deep neural networks to scale MORL to larger domains.
% Specifically, we use generalized versions of the Bellman equation~\cite{Bellman1957} to derive an algorithm for MORL with \textit{linear preferences}\footnote{If we have $n$ preferences, we assume their importance weights to lie on an $(n-1)$-simplex.} and provide proofs of their convergence. 
Our algorithm is based on two key insights -- (1) the optimality operator for a generalized version of Bellman equation~\cite{Bellman1957} with preferences is a valid contraction, and (2) optimizing for the convex envelope of multi-objective Q-values ensures an efficient alignment between preferences and corresponding optimal policies. We use hindsight experience replay~\cite{AndrychowiczCRS17} to re-use transitions for learning with different sampled preferences and homotopy optimization~\cite{watson1989modern} to ensure tractable learning. In addition, we also demonstrate how to use our trained model to automatically infer hidden preferences on a new task, when provided with just scalar rewards, through a combination of policy gradient and stochastic search over the preference parameters.

We perform empirical evaluation on four different domains -- deep sea treasure (a popular MORL benchmark), a fruit tree navigation task, task-oriented dialog, and the video game Super Mario Bros.
Our experiments demonstrate that our methods significantly outperform competitive baselines on all domains. 
For instance, our envelope MORL algorithm achieves an \textbf \% improvement on average user utility compared to the scalarized MORL in the dialog task and a factor {\bf 2x} average improvement on SuperMario game with random preferences. We also demonstrate that our agent can reasonably infer hidden preferences at test time using very few sampled trajectories. 
\section{Background}
\label{s:framework}
% \karthik{I'm inclined to merge this section with related work and call it 'Background and Related Work'? can explain outer and inner loop methods, etc. in context}

% \paragraph{Multi-Objective MDP}
\label{sec:momdp}
A multi-objective Markov decision process (MOMDP) can be represented by the tuple $\langle \mathcal S, \mathcal A, \mathcal P, \bm r, \Omega, f_{\bm \Omega} \rangle$  with state space $\mathcal S$, action space $\mathcal A$, transition distribution $\mathcal P (s' | s, a)$, vector reward function $\bm r(s, a)$, the space of preferences $\Omega$, and preference functions, e.g., $f_{\bm \omega}(\bm r)$ which produces a scalar \textit{utility} using preference $\bm \omega \in \Omega$. In this work, we consider the class of MOMDPs with linear preference functions, i.e., $f_{\bm \omega}(\bm r (s, a)) = \bm \omega^{\intercal}\bm r (s, a)$. We observe that if $\bm \omega$ is fixed to a single value, this MOMDP collapses into a standard MDP. On the other hand, if we consider all possible returns from an MOMDP, we have a Pareto frontier $\mathcal F^* := \{\bm {\hat r} \mid \not\exists \bm {\hat r}'\geq \bm {\hat r}\}$, 
% \xy{I am a little worried about the notation here: the return $\hat{r}$ is a random variable, but what we really mean by $\hat{r}$ here is all its possible values with positive probability density w.r.t. the Lebesgue measure.}
where the return $\bm {\hat r} := \sum_t\gamma^t \bm r(s_t,a_t)$. And for all possible preference in $\Omega$, we define a convex coverage set (CCS) of the Pareto frontier as: 
\[\mathtt{CCS} := \{\bm {\hat r} \in \mathcal F^* \mid \exists \bm \omega \in \Omega \textrm{ s.t. } \bm \omega^{\intercal}\bm {\hat r} \geq \bm \omega^{\intercal}\bm {\hat r}', \forall \bm {\hat r}' \in \mathcal F^*\},\] 
which contains all returns that provide the maximum cumulative utility.
% We use $\Pi^*_{\mathtt{CCS}}$ to denote the set of all corresponding policies.
Figure \ref{fig:ccs} (a) shows an example of CCS and the Pareto frontier.  The CCS is a subset of the Pareto frontier (points A to H, and K), containing all the solutions on its outer convex boundary (excluding point K). When a specific linear preference $\bm \omega$ is given, the point within the CCS with the largest projection along the direction of the relative importance weights will be the optimal solution (Figure~\ref{fig:ccs}(b)).
% Another key feature of an unknown linear preference scenario is that the linear preference is initially unrevealed, until a later phase of this scenario. Formally, the unknown linear preference scenario contains three consecutive phases: {\em learning phase}, {\em analysis phase}, and {\em execution phase}.

Our goal is to train an agent to recover policies for the \emph{entire CCS} of MOMDP and then adapt to the optimal policy for any given $\bm \omega \in \Omega$ at test time. We emphasize that we are not solving for a single, unknown $\bm \omega$, but instead aim for generalization across the entire space of preferences. Accordingly, our MORL setup has two phases:

\paragraph{Learning phase.} In this phase, the agent learns a set of {\em optimal policies} $\Pi_{\mathcal L}$ corresponding to the entire CCS of the MOMDP, using interactions with the environment and historical trajectories. For each $\pi \in \Pi_{\mathcal L}$, there exists at least one linear preference $\bm \omega$ such that no other policy $\pi'$ generates higher utility under that $\bm \omega$:
\begin{equation*}
    \pi \in \Pi_{\mathcal L} \Rightarrow \exists ~\bm \omega\in \Omega, \mbox{s.t.}~~ \forall\pi'\in \Pi, \bm \omega^{\intercal}    \bm v^{\pi}(s_0) \geq \bm \omega^{\intercal}    \bm v^{\pi'}(s_0),
\end{equation*}
where $s_0$ is a fixed initial state, and $\bm v^{\pi}$ is the value function, i.e., $\bm v^{\pi}(s) = \expe_{\pi}[\hat{\bm r}|s_0=s]$. Given any preference $\bm \omega$, $\Pi_{\mathcal L}(\bm \omega)$ determines the optimal policy.
% The agent in this phase is supposed to acquire a set of policies $\Pi_{\mathcal L} \subseteq \Pi_{\mathtt{CCS}}^*$ such that $\mathcal F_{\Pi_{\mathcal L}} = \mathtt{CCS}$, and $\Pi_{\mathcal L}(\bm \omega)$ determines a policy. Thus for any preference $\bm \omega$, one of the policies in  $\Pi_{\mathcal L}$ is optimal.

\paragraph{Adaptation phase.} After learning, the agent is provided a new task, with either a) a preference $\bm \omega$ specified by a human, or b) an unknown preference, where the agent has to automatically infer $\bm \omega$. 
% a task with unknown preference and has to automatically analyze the underlying preference using as few interactions as possible. This is similar to \textit{preference elicitation} by searching over the space of possible $\bm \omega$s. Alternatively, we may have a user specify a preference $\bm \omega$ based on the trade-off information the agent learned. During execution, the agent needs to respond with a policy $\Pi_{\mathcal L}(\bm \omega)$, using limited computational resources.
Efficiently aligning $\Pi_{\mathcal L}(\bm \omega)$ with the preferred optimal policy is non-trivial since the CCS can be very large. In both cases, the agent is evaluated on how well it can adapt to tasks with unseen preferences. 

\begin{figure}[t]
    \centering
    \includegraphics[width=1.0\textwidth]{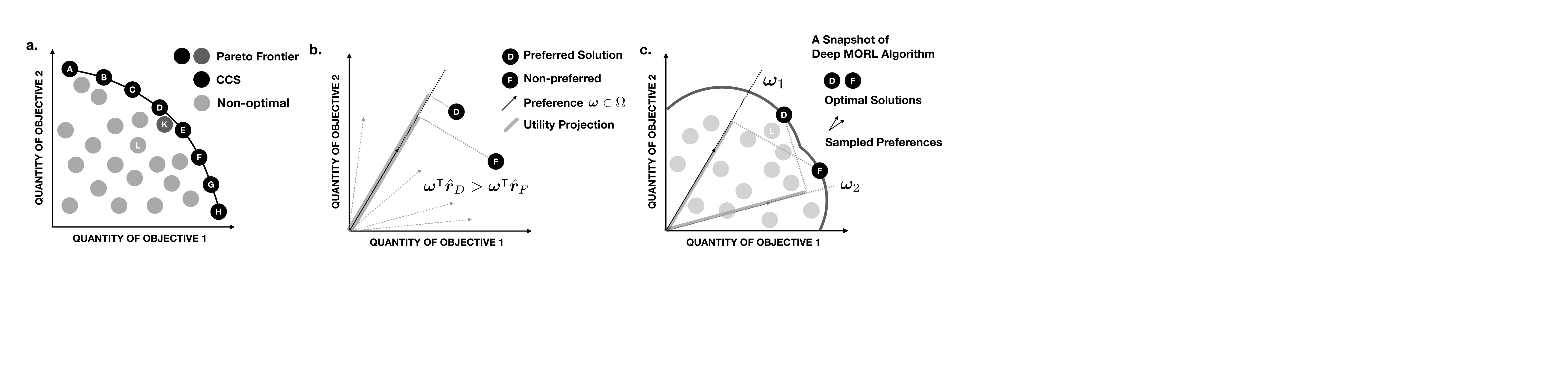}
    \vspace{-1em}
    \caption{\small (a) The Pareto frontier may encapsulate local concave parts (points A-H, plus point K), whereas CCS is a convex subset of Pareto frontier (points A-H). Point L indicates a non-optimal solution. (b) Linear preferences select the optimal solution from CCS with the highest utility, represented by the projection length along preference vector. Arrows are different linear preferences, and points indicate possible returns. Return D has better cumulative utility than return F under the preference in solid line. (c) The scalarized MORL algorithms (e.g., \cite{abels2018dynamic}) find the optimal solutions at a stage while they are not aligned with preference, e.g., two optimal solutions D and F in the CCS, misaligned with preferences $\bm \omega_2$ and $\bm \omega_1$.  The scalarized update cannot use the information of $\max_a Q(s,a,\bm \omega_1)$ (corresponding to F) to update the optimal solution aligned with $\bm \omega_2$ or vice versa. It only searches along $\bm \omega_{1}$ direction leading to non-optimal L, even if solution D has been seen under $\bm \omega_2$. It still requires many iterations for the value-preference alignment.}
    \label{fig:ccs}
    \vspace{-1em}
\end{figure}

% \section{Related Work}
\label{s:related_work}
% \karthik{maybe move related work to after section 3? might be easier to explain differences.}
\subsection{Related Work} 
% \karthik{move to later section?}
\paragraph{Multi-Objective RL} Existing MORL algorithms can be roughly divided into two main categories~\cite{VamplewDBID11,LiuXH15,RoijersVWD13}: {\em single-policy} methods and {\em multiple-policy} methods. Single-policy methods aim to find the optimal policy for a {\em given} preference among the objectives~\cite{MannorS01,TesauroDCKLRL07}. These methods explore different forms of preference functions, including non-linear ones such as the minimum over all objectives or the number of objectives that exceed a certain threshold. However, single-policy methods do not work when preferences are unknown.

\vspace{-0.5em}
Multi-policy approaches learn a {\em set} of policies to obtain the approximate Pareto frontier of optimal solutions. The most common strategy is to perform multiple runs of a single-policy method over different preferences~\cite{NatarajanT05,MoffaertDN13}. Policy-based RL algorithms \cite{PirottaPR15, ParisiPP17}  simultaneously learn the optimal manifold over a set of preferences.
Several value-based reinforcement learning algorithms employ an extended version of the Bellman equation and maintain the convex hull of the discrete Pareto frontier \cite{BarrettN08,HiraokaYM07,IimaK14}. Multi-objective fitted Q-iteration (MOFQI)~\cite{CastellettiPR11,CastellettiPR12} encapsulates preferences as input to a Q-function approximator and uses expanded historical trajectories to learn multiple policies. This allows the agent to construct the optimal policy for any given preference during testing. However, these methods explicitly maintain sets of policies, and hence are difficult to scale up to high-dimensional preference spaces. Furthermore, these methods are designed to work during the learning phase but cannot be easily adapted to new preferences at test time.

\vspace{-0.5em}
\paragraph{Scalarized Q-Learning.} Recent work has proposed the scalarized Q-learning algorithm~\cite{MossalamARW16} which uses a vector value function but performs updates after computing the inner product of the value function with a preference vector. This method uses an outer loop to perform a search over preferences, while the inner loop performs the scalarized updates. Recently, Abels et al.~\cite{abels2018dynamic} extended this to use a single neural network to represent value functions over the entire space of preferences. However, scalarized updates are not sample efficient and lead to sub-optimal MORL policies -- our approach uses a global optimality filter to perform envelope Q-function updates, leading to faster and better learning (as we demonstrate in Figure~\ref{fig:ccs}(c) and Section~\ref{s:results}). 
% \karthik{We need some more discussion on how exactly we differ from Abels. Can you add in some of the text from the rebuttal to better differentiate ourselves from Abels?}\runzhe{Three key contributions distinguish our work from Abels et al.~\cite{abels2018dynamic}:  (1) At algorithmic level, our envelope Q-learning algorithm utilizes the convex envelope of the solution frontier to update parameters of the policy network, which allows our method to quickly align one preference with optimal rewards and trajectories that may have been explored under other preferences. (2) At theoretical level, we introduce a theoretical framework for designing and analyzing value-based MORL algorithms, and convergence proofs for our envelope Q-learning algorithm. (3) At empirical level, we provide new evaluation metrics and benchmark environments for MORL. In terms of experiments, we apply our algorithm to a wider variety of domains including DST, FTN and two complex larger scale domains -- task-oriented dialog and supermario. Our FTN domain is a scaled up, more complex version of Minecart \cite{abels2018dynamic}}.

\vspace{-0.5em}
Three key contributions distinguish our work from Abels et al.~\cite{abels2018dynamic}:  (1) At algorithmic level, our envelope Q-learning algorithm utilizes the convex envelope of the solution frontier to update parameters of the policy network, which allows our method to quickly align one preference with optimal rewards and trajectories that may have been explored under other preferences. (2) At theoretical level, we introduce a theoretical framework for designing and analyzing value-based MORL algorithms, and convergence proofs for our envelope Q-learning algorithm. (3) At empirical level, we provide new evaluation metrics and benchmark environments for MORL and apply our algorithm to a wider variety of domains including two complex larger scale domains -- task-oriented dialog and supermario. Our FTN domain is a scaled up, more complex version of Minecart in \cite{abels2018dynamic}.

\vspace{-0.5em}
\paragraph{Policy Adaptation.} Our policy adaptation scheme is related to prior work in preference elicitation~\cite{conen2001preference,Boutilier02,chen2004survey} or inverse reinforcement learning~\cite{NgR00,PieterN04}. Inverse RL (IRL) aims to learn a scalar reward function from expert demonstrations, or directly imitate the expert's policy without intermediate steps for solving a scalar reward function~\cite{HoE16}. Chajewska et al.~\cite{ChajewskaKO01} proposed a Bayesian version to learn the utility function. IRL is effective when the hidden preference is fixed and expert demonstrations are available. In contrast, we require policy adaptation across various different preferences and do not use any demonstrations.

\section{Multi-objective RL with Envelope Value Updates}
% \label{s:methods}
\label{sec:envelope}

% \karthik{TODO: talk about limitations of current scalarized methods, what are they missing. Challenges and opportunities of using vectorized updates. We explore vectorized updates, prove convergence, provide new algo}

% While scalarized Q-learning produces reasonable policies for MORL tasks, they are not sample efficient in dealing with the exponential blow-up in value space due to the added dimension of preferences ($\bm \omega$).

In this section, we propose a new algorithm for multi-objective RL called \emph{envelope Q-learning}. Our key idea is to use vectorized value functions and perform \textit{envelope updates}, which utilize the convex envelope of the solution frontier to update parameters. This is in contrast to approaches like \textit{scalarized Q-Learning}, which perform value function updates using only a single preference at a time. Since we learn a set of policies simultaneously over multiple preferences, and our concept of optimality is defined on vectorized rewards, existing convergence results from single-objective RL no longer hold. Hence, we first provide a theoretical analysis of our proposed update scheme below followed by a sketch of the resulting algorithm.
% \todo{@Karthik: emphasize challenge a bit more.}

\paragraph{Bellman operators.}
The standard Q-Learning~\cite{WatkinsD92} algorithm for single-objective RL utilizes the Bellman optimality operator $T$:
\begin{equation}
\label{eq:single-obj-eval}
    (T Q)(s, a) := r(s, a) + \gamma \mathbb E_{s'\sim\mathcal P(\cdot | s,a)}(H Q)(s').
\end{equation}
where the operator $H$ is defined by $(H  Q)(s') := \sup_{a'\in\mathcal A}  Q(s', a')$ is an optimality filter over the Q-values for the next state $s'$.

We extend this to the MORL case by considering a value space $\mathcal Q \subseteq (\Omega \rightarrow \mathbb R^m)^{\mathcal S \times \mathcal A}$, containing all bounded functions $\bm Q(s,a,\omega)$ -- estimates of expected total rewards under $m$-dimensional preference ($\bm \omega$) vectors. We can define a corresponding value metric $d$ as:
\begin{equation}
    d(\bm Q, \bm Q') := \sup_{\substack{s\in\mathcal S, a\in\mathcal A\\ \bm \omega\in\Omega}} |\bm \omega^{\intercal} (\bm Q(s, a, \bm \omega) - \bm Q'(s, a, \bm \omega))|.
\end{equation}
Since the identity of indiscernibles~\cite{naylor2000linear} does not hold, we note that $d$ forms a complete pseudo-metric space, and refer to  $\bm Q$ as a \textit{Multi-Objective Q-value (MOQ) function}.
Given a policy $\pi$ and sampled trajectories $\tau$, we first define a multi-objective evaluation operator $\mathcal T_{\pi}$ as:
\begin{equation}
    (\mathcal T_{\pi}\bm Q)(s,a, \bm \omega) := \bm r(s,a) + \gamma \mathbb E_{\tau \sim (\mathcal P, \pi)} \bm Q(s', a', \bm \omega).
\end{equation}
% where $\tau$ is a sampled trajectory using policy $\pi$.

We then define an optimality filter $\mathcal H$ for the MOQ function as $(\mathcal H \bm Q)(s, \bm \omega) := \arg_{Q}\sup_{a\in\mathcal A, \bm \omega' \in \Omega} \bm \omega^{\intercal}\bm Q(s, a, \bm \omega')$, where the $\arg_{Q}$ takes the multi-objective value corresponding to the supremum (i.e., $\bm Q(s, a, \bm \omega')$ such that $(a, \bm \omega') \in \arg\sup_{a\in\mathcal A, \bm \omega' \in \Omega} \bm \omega^{\intercal}\bm Q(s, a, \bm \omega')$). 
% \runzhe{The return of $\arg_{Q}$ depends on which $\bm \omega$ is chosen for scalarization, and we keep $\arg_{Q}$ for simplicity.}
The return of $\arg_{Q}$ depends on which $\bm \omega$ is chosen for scalarization, and we keep $\arg_{Q}$ for simplicity.
This can be thought of as generalized version of the single-objective optimality filter in Eq.~\ref{eq:single-obj-eval}. Intuitively, $\mathcal H$ solves the convex envelope (hence the name envelope Q-learning) of the current solution frontier to produce the $\bm Q$ that optimizes utility given state $s$ and preference $\bm \omega$. This allows for more optimistic Q-updates compared to using just the standard Bellman filter ($H$) that optimizes over actions only -- this is the update used by scalarized Q-learning~\cite{abels2018dynamic}.
We can then define a \textit{multi-objective optimality operator} $\mathcal T$ as:
\begin{equation}
    (\mathcal T \bm Q)(s, a, \bm \omega) := \bm r(s, a) + \gamma \mathbb E_{s'\sim\mathcal P(\cdot | s,a)}(\mathcal H \bm Q)(s', \bm \omega).
\end{equation}
% This operator is basically a vectorized form of the single-objective one (Eq.~\ref{eq:single-obj-eval}), along with the convex

The following theorems demonstrate the feasibility of using our optimality operator for multi-objective RL. Proofs for all the theorems are provided in the supplementary material.
\begin{thm}[Fixed Point of Envelope Optimality Operator]
    Let $\bm Q^* \in \mathcal Q$ be the preferred optimal value function in the value space, such that 
    \begin{equation}
        \bm Q^*(s, a, \bm \omega) = \arg_{Q} \sup_{\pi\in \Pi}\bm \omega^{\intercal}\mathbb E_{\substack{\tau\sim (\mathcal P, \pi) \\| s_0=s, a_0=a}}\left[\sum_{t=0}^{\infty}\gamma^t \bm r(s_t,a_t) \right],
    \end{equation} 
    where the $\arg_{Q}$ takes the multi-objective value corresponding to the supremum. Then, $\bm Q^* = \mathcal T \bm Q^*$.
    \label{thm:fix-envelope}
\end{thm}
Theorem \ref{thm:fix-envelope} tells us the preferred optimal value function is a fixed-point of $\mathcal T$ in the value space. 

\begin{thm}[Envelope Optimality Operator is a Contraction]
    Let $\bm Q, \bm Q'$ be any two multi-objective Q-value functions in the value space $\mathcal Q$ as defined above. Then, the Lipschitz condition $d(\mathcal T \bm Q, \mathcal T \bm Q') \leq \gamma d(\bm Q, \bm Q')$ holds, where $\gamma \in [0, 1)$ is the discount factor of the underlying MOMDP $\mathcal M$.
    \label{thm:contraction-envelope}
\end{thm}
% Since the value distance $d$ is a pseudo-metric, iteratively applying contraction $\mathcal T$ may not necessarily to a desired fixed point. 
%\karthik{might be worth using different symbols for the optimality operator, etc. in naive vs envelope?}

Finally, we provide a generalized version of Banach's Fixed-Point Theorem in the pseudo-metric space.

% \paragraph{Generalized Fixed-Point Theorem}
\label{sec:envelope:theory}
\begin{thm}[Multi-Objective Banach Fixed-Point Theorem]
    If $\mathcal T$ is a contraction mapping with Lipschitz coefficient $\gamma$ on the complete pseudo-metric space $\langle \mathcal Q, d \rangle$, and $\bm Q^*$ is defined as in Theorem \ref{thm:fix-envelope}, then $\lim_{n\rightarrow\infty} d(\mathcal T^n \bm Q, \bm Q^*) = 0$ for any $\bm Q\in \mathcal Q$.
    \label{thm:generalized-fix}
\end{thm}
% \karthik{Should we call this multi-objective Banach Thm?}

Theorems \ref{thm:fix-envelope}-\ref{thm:generalized-fix} guarantee that iteratively applying optimality operator $\mathcal T$ on any MOQ-value function will terminate with a function $\bm Q$ that is equivalent to $\bm Q^*$ under the measurement of pseudo-metric $d$. These $\bm Q$s are as good as $\bm Q^*$ since they all have the same utilities for each $\bm \omega$, and will only differ when the utility corresponds to a recess in the frontier (see Figure \ref{fig:ccs}(c) for an example, at the recess, either D or F is optimal).
% \karthik{Is this because we assume convexity for this proof?}

Maintaining the envelope $\sup_{{\bm \omega'}} {\bm \omega^{\intercal}}Q(\cdot,\cdot,{\bm \omega'})$ allows our method to quickly align one preference with optimal rewards and trajectories that may have been explored under other preferences, while scalarized updates that optimizes the scalar utility cannot use the information of $\max_a Q(s,a, {\bm \omega'})$ to update the optimal solution aligned with a different ${\bm \omega}$.
As illustrated in Figure 2 (c), assuming we have found two optimal solutions D and F in the CCS, misaligned with preferences $\bm \omega_2$ and $\bm \omega_1$. The scalarized update cannot use the information of $\max_a Q(s,a,\bm \omega_1)$ (corresponding to F) to update the optimal solution aligned with $\bm \omega_2$ or vice versa. It only searches along $\bm \omega_{1}$ direction leading to non-optimal L, even if solution D has been seen under $\bm \omega_2$. Hence, the envelope updates can have better sample efficiency in theory, as is also seen from the empirical results.

\begin{wrapfigure}{R}{0.53\textwidth}
{
\centering
\begin{minipage}[c]{0.53\textwidth}
\vspace{-1em}
\begin{algorithm}[H] %[t]
\small
    \caption{Envelope MOQ-Learning}
    \label{algo:envelope}
    \KwIn{a preference sampling distribution $\mathcal D_{\omega}$,  path $p_{\lambda}$ for the balance weight $\lambda$ increasing from 0 to 1.}
    Initialize replay buffer $\mathcal D_{\tau}$, network $\bm Q_{\theta}$, and $\lambda = 0$.\\
    \For{episode = $1, \dots, M$} {
        Sample a linear preference $\bm \omega \sim \mathcal D_{\omega}$.\\
        \For{t = $0, \dots, N$} {
        	Observe state $s_t$.\\
            Sample an action $\epsilon$-greedily:
%                \vspace{-1em}
                \[a_t = \begin{cases}
                    \textrm{random action in } \mathcal A, & \textrm{w.p. ~} \epsilon;\\ 
                    \max_{a\in\mathcal A} \bm \omega^{\intercal} \bm Q(s_t, a, \bm \omega; \theta), & \textrm{w.p ~} 1-\epsilon.
                \end{cases}\]\\
%				\vspace{-1em}
            Receive a vectorized reward $\bm r_t$ and observe $s_{t+1}$.\\
            Store transition $(s_t, a_t, \bm r_t, s_{t+1})$ in $\mathcal D_{\tau}$.\\
            \If{update}{
                Sample $N_{\tau}$ transitions $(s_j, a_j, \bm r_j, s_{j+1})\sim\mathcal D_{\tau}$.\\
                Sample $N_{\omega}$ preferences $W = \{\bm \omega_{i}\sim \mathcal D_{\omega}\}$.\\
                Compute $y_{ij} = (\mathcal T \bm Q)_{ij} = $
                {\small
                \[\begin{cases}
                    \bm r_j, \hfill \hfill \hfill \hfill \hfill \hfill \hfill \hfill \hfill \hfill \hfill   
                    \textrm {for terminal~} s_{j+1}; \\
                    \bm r_j + \gamma \arg_{Q}\displaystyle \max_{\substack{a\in\mathcal A, \\ \bm \omega'\in W}} \bm \omega^{\intercal}_i\bm Q(s_{j+1}, a, \bm \omega'; \theta), \hfill \textrm {o.w.}
                \end{cases}\]
                }
                for all $ 1 \leq i \leq N_{\omega}$ and  $1 \leq j \leq N_{\tau}$.\\
%                \vspace{-1em}
                Update $Q_{\theta}$ by descending its stochastic gradient according to equations \ref{eq:enve-lossA} and \ref{eq:enve-lossB}:
%				\vspace{-1em}
                \begin{equation*}
                    \nabla_{\theta}L(\theta) = (1-\lambda) \cdot \nabla_{\theta}L^{\mathtt A}(\theta) + \lambda \cdot \nabla_{\theta}L^{\mathtt B}(\theta).
                \end{equation*}\\
                Increase $\lambda$ along the path $p_{\lambda}$.\\
            }
        }
    }
\end{algorithm}
\end{minipage}
% \par
}
\vspace{-20pt}
\end{wrapfigure}
% \vspace{-2em}
\paragraph{Learning Algorithm.} 
\label{sec:envelope-update}
Using the above theorems, we provide a sample-efficient learning algorithm for multi-objective RL (Algorithm~\ref{algo:envelope}). Since our goal is to induce a single model that can adapt to the entire space of $\Omega$, we use one parameterized function to represent $\mathcal Q \subseteq (\Omega \rightarrow \mathbb R^m)^{\mathcal S \times \mathcal A}$. We achieve this by using a deep neural network with $s, \bm\omega$ as input and $|\mathcal A|\times m$ Q-values as output.
%\karthik{@Runzhe: is this consistent?} 
We then minimize the following loss function at each step $k$:\footnote{We use double Q learning with target Q networks following Mnih et al.~\cite{MnihKSRVBGRFOPB15}}
\begin{equation}
    L^{\mathtt A} (\theta) = \mathbb E_{s, a, \bm \omega} \left[ \left\| \bm y - \bm Q(s, a, \bm \omega; \theta) \right\|_2^2 \right],
    \label{eq:enve-lossA}
\end{equation}
where ${\bm y = \expe_{s'}[\bm r + \gamma \arg_{Q}\max_{a, \bm \omega'} \bm \omega^{\intercal}\bm Q(s', a, \bm \omega'; \theta_k)}]$, which empirically can be estimated by sampling transition $(s, a, s', r)$ from a replay buffer.

% \runzhe{should we explain subscript $i$ and $j$ in the text? I made it clearer in the algorithm 1}

Optimizing $L^{\mathtt A}$ directly is challenging in practice because the optimal frontier contains a large number of discrete solutions, which makes the landscape of loss function considerably non-smooth.
%\karthik{@Runzhe: add in the reason here}. 
% Minimizing the loss function $L^{\mathbb A}$ is trying to drag the vector $\bm Q$ close to $\mathcal T \bm Q$. This ensures the correctness of our algorithm, to predict a $\bm Q$ as the real solution, while this means the square error is hard to be optimized in practice. 
To address this, we use an auxiliary loss function $L^{\mathtt B}$:
\begin{equation}
    L^{\mathtt B}(\theta) = \mathbb E_{s,a,\bm \omega}[\left|\bm \omega^{\intercal}\bm y - \bm \omega^{\intercal} \bm Q(s,a,\bm \omega;\theta)\right|].
    \label{eq:enve-lossB}
\end{equation}

Combined, our final loss function is $L(\theta) = (1-\lambda)\cdot L^{\mathtt A}(\theta) + \lambda \cdot L^{\mathtt B}(\theta)$, where $\lambda$ is a weight to trade off between losses $L^{\mathtt A}$ and $L^{\mathtt B}_k$. We slowly increase the value of $\lambda$ from $0$ to $1$, to shift our loss function from  $L^{\mathtt A}$ to  $L^{\mathtt B}$. This method, known as {\em homotopy optimization}~\cite{watson1989modern}, is effective since for each update step, it uses the optimization result from the previous step as the initial guess. $L^{\mathtt A}$ first ensures the prediction of $\bm Q$ is close to any real expected total reward, although it may not be optimal. $L^{\mathtt B}$ provides an auxiliary pull along the direction with better utility.  

The loss function above has an expectation over $\bm \omega$ -- this entails sampling random preferences in the algorithm. However, since the $\bm \omega$s are decoupled from the transitions, we can increase sample efficiency by using a scheme similar to Hindsight Experience Replay~\cite{AndrychowiczCRS17}.
% When updating the utility-based multi-objective Q-network \runzhe{utility-based MOQ-net is used in scalarized algo.} according to equation \ref{eq:naive-dqn}, we associate each sampled transition $(s_t^i, a_t^i, \bm r_t^i, s_{t+1}^i)$ from replay buffer $\mathcal D_{\tau}$ with $N_{\omega}$ preferences $\{\bm \omega_{1}, \bm \omega_{2}, \dots, \bm \omega_{N_{\omega}}\}$ sampled from $\mathcal D_{\omega}$. 
Furthermore, computing the optimality filter $\mathcal H$ over the entire $\mathcal Q$ is infeasible; instead we approximate this by applying $\mathcal H$ over a minibatch of transitions before performing parameter updates. Further details on our model architectures and implementation details are available in the supplementary material (Section~\ref{app:sec:envelope:update}).

\paragraph{Policy adaptation.}
\label{s:inference}

Once we obtain a policy model $\Pi_{\mathcal L}(\bm \omega)$ from the learning phase, the agent can adapt to any provided preference by simply feeding the $\bm \omega$ into the network. While this is a straightforward scenario, we also consider a more challenging test where only scalar rewards are available and the agent has to uncover a hidden preference $\bm \omega$ while adapting to the new task. For this case, we assume preferences are drawn from a truncated multivariable Gaussian distribution $\mathcal D_{\omega}^m(\mu_1,  \dots, \mu_m; \sigma)$ on an $(m-1)$-simplex, where nonnegative parameters $\mu_1, \dots, \mu_m$ are the means with $\mu_1+\cdots +\mu_m = 1$, and $\sigma$ is a fixed standard deviation for all dimensions. Our goal is then to infer the parameters of this Gaussian distribution, for which we perform a combination of policy gradient (e.g., REINFORCE~\cite{Williams92}) and stochastic search while keeping the policy model fixed. We determine the best preference parameters that maximize the expected return in the target task:
\begin{equation}
\argmax_{\mu_1, \dots,\mu_m}\mathbb E_{\bm \omega \sim \mathcal D_{\omega}^m} \left[\mathbb E_{\tau \sim (\mathcal P, \Pi_{\mathcal L}(\bm \omega)) }\left[\sum_{t=0}^{\infty}\gamma^t r_t(s_t,a_t)\right]\right].
\end{equation}

% This optimization problem can be efficiently solved by stochastic search with policy gradient algorithm such as REINFORCE \cite{Williams92} in few iterations, since there are only $m$ parameters. 
% Starting with an initial guess of $\mu_1,\cdots,\mu_m$, we sample a few trajectories to estimate the expected returns and update the parameters to maximize the objective.
% Since the deviation $\sigma$ is small, the algorithm infers the final  preference as $\bm \omega^* = (\mu_1^*, \dots, \mu_m^*)$.

% In each iteration, we use few samples to estimate the expected returns under certain preference $\hat R_{\omega}  \approx \mathbb E_{\tau \sim (\mathcal P, \Pi_{\mathcal L}(\bm \omega)) }\left[\sum_{t=0}^{\infty}[\gamma^t r_t(s_t,a_t)\right]$, update $(\mu_1,\dots,\mu_m)$ by increasing it with a small step $\eta$, where

% \[\eta\propto \mathbb E_{\bm \omega \sim \mathcal D_{\omega}^m} \left[\nabla_{\mu_1, \dots, \mu_m} \log \mathtt{Pr}_{\mathcal D_{\omega}^{m}}[\bm \omega]\cdot\hat R_{\omega}\right],\]
% and then project it back to $(m-1)$-simplex. 

\section{Experiments}
\label{s:experiments}
% We first describe two quantitative evaluation metrics we use and then demonstrate our experimental results on two synthetic domains, Deep Sea Treasure (DST) and Fruit Tree Navigation (FTN), as well as two complex real domains, Dialog Policy Learning and SuperMario Game. We leave specific model architecture and other domain and implementation details to the supplementary material.
% \subsection{Experiment
\paragraph{Evaluation Metrics.}
\label{sec:metrics}
% \karthik{emphasize that this is evaluation on test, not train tasks}
% We use two quantitative metrics to evaluate the empirical performance of our algorithms:
Three metrics are
to evaluate the empirical performance on test tasks:

\begin{wrapfigure}{R}{0.55\textwidth}
    \vspace{-1.0em}
    \centering
    \includegraphics[width=0.55\columnwidth]{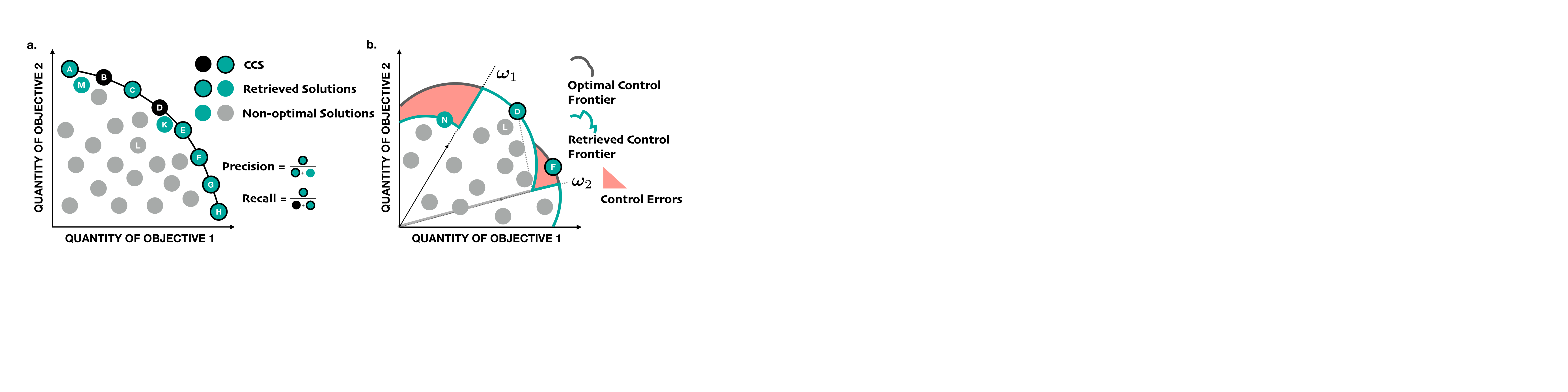}
    \vspace{-0.2em}
    \caption{Illustration of evaluation metrics for MORL. (a.) Coverage ratio (CR) measures an agent's ability to find all the potential optimal solutions in the convex coverage set of Pareto frontier. Dots with black boundary are solutions in CCS, dots without black boundary are non-optimal returns, and dots in green are solutions retrieved by an MORL algorithm. CR is the F1 based on the precision and recall calculation. (b.) Adaptation error (AE) measures an agent's ability of policy adaptation to real-time specified preferences. The gray curve indicates the theoretical limit of the best cumulative utilities under all preference, and the green curve indicates the cumulative utilities of an MORL algorithm. AE is the average gap between these two curves over all preferences.}
  	\vspace{-0.2em}
    \label{fig:eval-metric}
\end{wrapfigure}

a) \textit{Coverage Ratio (CR).}
\label{sec:eval-metic:cr}
The first metric is \textit{coverage ratio} (CR), which evaluates the agent's ability to recover optimal solutions in the convex coverage set ($\mathtt{CCS}$). If $\mathcal F \subseteq \mathbb R^{m}$ is the set of solutions found by the agent (via sampled trajectories),
% and $\mathcal F_{Q} \subseteq \mathbb R^{m}$ is the set of agent's guessed solutions (via predictions from Q-network), 
we define $\mathcal F \cap_{\epsilon}\mathtt{CCS}:=\{x\in \mathcal F \mid \exists y \in \mathtt{CCS} \textrm{~s.t.~} \|x-y\|_{1}/\|y\|_1 \leq \epsilon \}$ as the intersection between these sets with a tolerance of $\epsilon$. 
The CR is then defined as:
\begin{equation}
    \mathtt{CR_{F1}}(\mathcal F) := \mathtt{2} \cdot \frac{ \mathtt{precision\cdot recall}}{\mathtt{precision} + \mathtt{recall}},
\end{equation}
where the $\mathtt{precision} = |\mathcal F \cap_{\epsilon} \mathtt{CCS}|/|\mathtt{\mathcal F}|$, indicating the fraction of optimal solutions among the retrieved solutions, and the $\mathtt{recall} = |\mathcal F \cap_{\epsilon}\mathtt{CCS}|/|\mathtt{CCS}|$, indicating the fraction of optimal instances that have been retrieved over the total amount of optimal solutions (see Figure~\ref{fig:eval-metric}(a)).
% The F1 score is their harmonic mean. 
% In our evaluation of all the three synthetic tasks, we set $\epsilon = 0.00$ for $\mathcal F_{\Pi_{\mathcal L}}$ and $\epsilon = 0.20$ for $\mathcal F_{Q}$. 
% Figure \ref{fig:eval-metric} (a.) illustrates an example of the computation of coverage ratio.

b) \textit{Adaptation Error (AE).}
\label{sec:eval-metic:aq}
Our second metric compares the retrieved control frontier with the optimal one, when an agent is provided with a specific preference $\bm\omega$ during the adaptation phase:
\begin{equation}
    \mathtt{AE}(\mathcal C) := \mathbb E_{\bm \omega \sim \mathcal D_{\omega}} [|\mathcal C(\bm \omega) - \mathcal C_{\textrm{opt}}(\bm \omega)| / \mathcal C_{\textrm{opt}}(\bm \omega)],
\end{equation}
% \karthik{Fix equation above}
which is the expected relative error between optimal control frontier $\mathcal C_{\textrm{opt}}: \Omega \rightarrow \mathbb R$ with $\bm \omega \mapsto \max_{\bm{\hat r}\in \mathtt{CCS}} \bm \omega^{\intercal} \bm{\hat r}$ and the agent's control frontier $\mathcal C_{\pi_{\bm \omega}} = \bm \omega^{\intercal} \bm{\hat r}_{\pi_{\bm \omega}}$. 

c) \textit{Average Utility (UT).} This measures the average utility obtained by the trained agent on randomly sampled preferences and is a useful proxy to AE when we don't have access to the optimal policy.

\paragraph{Domains.}
% \karthik{Can we itemize this just like the baselines?}
% We test our algorithms in four experimental domains: Two synthetic tasks, Deep Sea Treasure and Fruit Tree Navigation, allow us to concretely measure performance of our algorithms and compare with baselines since we have complete knowledge of the optimal set of policies. We also test on the other two complex domains, Task-Oriented Dialog Policy Learning and SuperMario Game to further examine the effectiveness of our algorithms. We list the task descriptions below and leave the domain details and illustrations to the supplementary material.
We evaluate on four different domains (complete details in supplementary material):

\begin{enumerate}[leftmargin=0.2in]
    \item {\bf Deep Sea Treasure (DST)} A classic MORL benchmark~\cite{VamplewDBID11} in which an agent controls a submarine searching for treasures in a $10 \times 11$-grid world while trading off {\tt time-cost} and {\tt treasure-value}. The grid world contains 10 treasures of different values. Their values increase as their distances from the starting point $s_0= (0, 0)$ increase. We ensure the Pareto frontier of this environment to be convex.
    \item {\bf Fruit Tree Navigation (FTN)} A full binary tree of depth $d$ with randomly assigned vectorial reward $\bm r\in \mathbb R^6$ on the leaf nodes. These rewards encode the amounts of six different components of nutrition of the {\em fruits} on the tree: $\{\mathtt{Protein}, \mathtt{Carbs}, \mathtt{Fats}, \mathtt{Vitamins}, \mathtt{Minerals}, \mathtt{Water}\}$. For every leaf node, $\exists \bm \omega$ for which its reward is optimal, thus all leaves lie on the $\mathtt{CCS}$. The goal of our MORL agent is to find a path from the root to a leaf node that maximizes utility for a given preference, choosing between left or right subtrees at every non-terminal node.
    \item {\bf Task-Oriented Dialog Policy Learning (Dialog)}  A modified task-oriented dialog system in the restaurant reservation domain based on PyDial~\cite{UltesRSVKCBMWGY17}. We consider the task success rate and the dialog brevity (measured by number of turns) as two competing objectives of this domain.
    \item {\bf Multi-Objective SuperMario Game (SuperMario)} A multi-objective version of the popular video game Super Mario Bros. We modify the open-source environment from OpenAI gym~\cite{gym-super-mario-bros} to provide vectorized rewards encoding five different objectives: {\tt x-pos}: value corresponding to the difference in Mario's horizontal position between current and last time point,  {\tt time}: a small negative time penalty, {\tt deaths}: a large negative penalty given each time Mario dies
    % \footnote{Mario has up to three lives in one episode.}
    ,  {\tt coin}:  rewards for collecting coins, and {\tt enemy}: rewards for eliminating an enemy. 
\end{enumerate}

\paragraph{Baselines.}  We compare our envelope MORL algorithm with classic and state-of-the-art baselines: 
\begin{enumerate}[leftmargin=0.2in]
    \item {\bf MOFQI} \cite{CastellettiPR12}: Multi-objective fitted Q-iteration where the Q-approximator is a large linear model.
    \item {\bf CN+OLS} \cite{abels2018dynamic}: Conditional neural network with Optimistic Linear Support (OLS) method as the outer loop for selecting $\bm \omega$. This method is first proposed in \cite{MossalamARW16} with multiple neural networks, and we employ an improved version using single conditional neural network \cite{abels2018dynamic}. 
    \item {\bf Scalarized} \cite{abels2018dynamic}: The state-of-the-art algorithm uses scalarized Q-update with double Q-learning, prioritized and hindsight experience replay, which is equivalent to CN+DER proposed in \cite{abels2018dynamic}.
\end{enumerate}

\paragraph{Main Results.}
\label{s:results}

% \begin{table}[t]
%     \centering
%     \caption{Performance comparison of different MORL algorithm in learning and adaptation phases across four experimental domains. The numbers of training episodes are the same for all methods.}
%     \resizebox{\columnwidth}{!}{
%     \begin{tabular}{ccccccc}
%     	\toprule
% 		\multirow{2}{3em}{Method} &  \multicolumn{2}{c}{DST} 
% 		       & \multicolumn{2}{c}{FTN ($d=6$)} 
% 		       & \multicolumn{1}{c}{Dialog} 
% 		       & \multicolumn{1}{c}{SuperMario}\\
% 		\cmidrule(lr){2-3}\cmidrule(lr){4-5}\cmidrule(lr){6-6}\cmidrule(lr){7-7}
% 		{} & {\tt CR} & {\tt AQ} ($\alpha=0.01$) & {\tt CR} & {\tt AQ} ($\alpha=10.0$) & {\tt Avg.}{\tt UT} & {\tt Avg.}{\tt UT}\\
% 		\midrule
% 		MOFQI &  
% 		0.639 $\pm$ 0.421 & 0.417 $\pm$ 0.134 &
% 		0.197 $\pm$ 0.000 & 0.362 $\pm$ 0.001 &
% 		2.17 $\pm$ 0.21 & --
%         \\
% 		CN+OLS&
% 		0.751 $\pm$ 0.163 & 0.743 $\pm$ 0.008 &
% 		-- & -- &
% 		2.53 $\pm$ 0.22$^*$ & --
%         \\
% 		Scalarized &
% 		0.989 $\pm$ 0.024 & 0.992 $\pm$ 0.005 &
% 		0.914 $\pm$ 0.044 & 0.859 $\pm$ 0.035 &
% 		2.38 $\pm$ 0.22 & 301.7 $\pm$ 49.2
%         \\
% 		Envelope (ours) &
% 		{\bf 0.994 $\pm$ 0.001} & {\bf 0.998 $\pm$ 0.000} &
% 		{\bf 0.987 $\pm$ 0.021} & {\bf 0.940 $\pm$ 0.011} &
% 		{\bf 2.65 $\pm$ 0.22} & {\bf 319.7 $\pm$ 34.4}
%         \\
% 		\bottomrule
%     \end{tabular}
%     }
%     \label{tab:overall}
% \end{table}

\begin{table}[t]
    \centering
    
    \resizebox{\columnwidth}{!}{
        \begin{tabular}{ccccccc}
    	\toprule
		\multirow{2}{3em}{Method} &  \multicolumn{2}{c}{DST} 
		       & \multicolumn{2}{c}{FTN ($d=6$)} 
		       & \multicolumn{1}{c}{Dialog$^2$} 
		       & \multicolumn{1}{c}{SuperMario$^2$}\\
		\cmidrule(lr){2-3}\cmidrule(lr){4-5}\cmidrule(lr){6-6}\cmidrule(lr){7-7}
		{} & {\tt CR $\uparrow$ } & {\tt AE $\downarrow$ } & {\tt CR $\uparrow$} & {\tt AE $\downarrow$} & {\tt Avg.}{\tt UT} $\uparrow$ & {\tt Avg.}{\tt UT} $\uparrow$\\
		\midrule
		MOFQI &  
		0.639 $\pm$ 0.421 & 139.6 $\pm$ 25.98 &
		0.197 $\pm$ 0.000 & 0.176 $\pm$ 0.001 &
		2.17 $\pm$ 0.21 & --
        \\
		CN+OLS&
		0.751 $\pm$ 0.163 & 34.63 $\pm$ 1.396 &
		-- & -- &
		2.53 $\pm$ 0.22 & --
        \\
		Scalarized &
		0.989 $\pm$ 0.024 & 0.165 $\pm$ 0.096 &
		0.914 $\pm$ 0.044 & 0.016 $\pm$ 0.005 &
		2.38 $\pm$ 0.22 & 162.7 $\pm$ 77.66
        \\
		Envelope (ours)$^1$ &
		{\bf 0.994 $\pm$ 0.001} & {\bf 0.152 $\pm$ 0.006} &
		{\bf 0.987 $\pm$ 0.021} & {\bf 0.006 $\pm$ 0.001} &
		{\bf 2.65 $\pm$ 0.22} & {\bf 321.2 $\pm$ 146.9}
        \\
		\bottomrule
    \end{tabular}
    }
    \vspace{5pt}
    \caption{
    % \karthik{Can we add up and down arrows next to CR and AE to say higher is better and lower is better for their scores?} 
    Comparison of different MORL algorithms in learning and adaptation phases across four experimental domains. $\uparrow$ indicates higher is better, and $\downarrow$ indicates lower is better for the scores. Each data point indicates the mean and standard deviation over 5 independent training and test runs. $^1$Using the unpaired t-test, we obtain significance scores of $p<0.05$ vs MOFQI on all domains, $p<0.01$ vs CN+OLS on DST and $p<0.05$ vs Scalarized on FTN, Dialog and SuperMario. $^2$Additional results are in the supplementary material \ref{supp:sec:dialog} and  \ref{supp:sec:mario}. 
    % \karthik{add link to section in appendix}
    }
    % \vspace{-1.0em}
    \label{tab:overall}
    \vspace{-10pt}
\end{table}

Table~\ref{tab:overall} shows the performance comparison of different MORL algorithms in four domains. We elaborate training and test details for each domain in supplementary material. In DST and FTN we compare CR and AE as defined in section~\ref{sec:metrics}. In the task-oriented dialog policy learning task, we compare the average utility (Avg. UT) for 5,000 test dialogues with uniformly sampled user preferences %\todo{how many prefs?} 
on success and brevity. In the SuperMario game, the Avg. UT is over 500 test episodes with uniformly sampled preferences. %\todo{how many prefs?}
The envelope algorithm steadily achieves the best performance in terms of both learning and adaptation among all the MORL methods in all four domains. 

\begin{figure*}[h]
	\centering
	\includegraphics[width=0.98\textwidth]{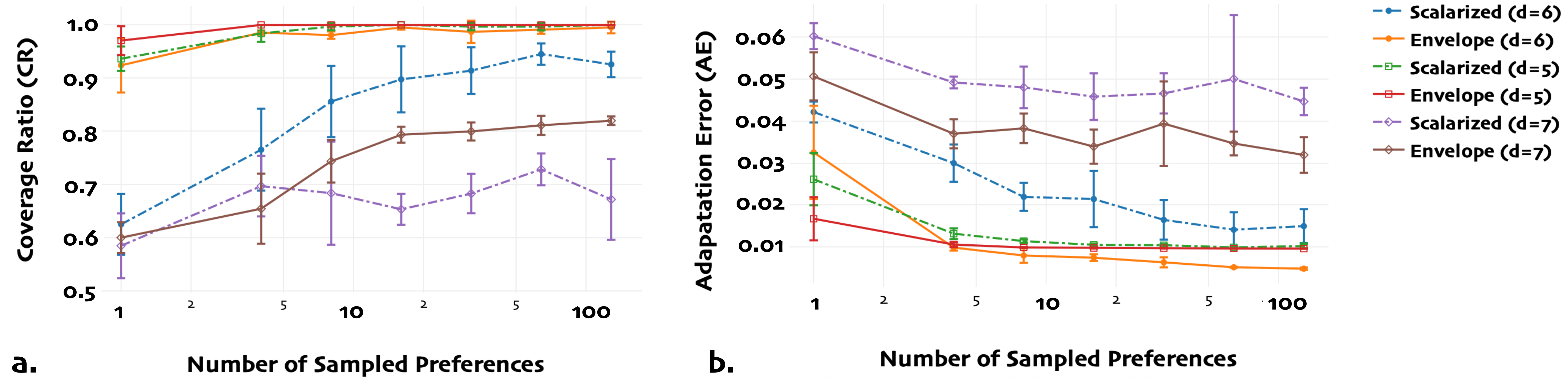}
% 	\vspace{-0.5em}
	\caption{\small Coverage Ratio (CR) and Adaptation Error (AE) comparison of the scalarized algorithm~\cite{abels2018dynamic} and our envelope deep MORL algorithm over 5000 episodes of FTN tasks of depths $d= 5, 6, 7$. Higher CR indicates better coverage of optimal policies, lower AE indicates better adaptation. The error bars are standard deviations of CR and AE estimated from 5 independent runs under each configuration.}
%	\karthik{make labels bigger. }
	\vspace{-1em}
	\label{fig:sample-ft}
\end{figure*}

\paragraph{Scalability.} There are three aspects of the scalability of a MORL algorithm: the ability to deal with (1) large state space, (2) many objectives, and (3) large optimal policy set. 
Unlike other neural network-based methods, MOFQI cannot deal with the large state space, e.g., the video frames in SuperMario Game. The CN+OLS baseline requires solving all the intersection points of a set of hyper-planes thus is computationally intractable in domains with $m>3$ objectives, such as FTN and SuperMario. We denote these entries as ``--" in Table \ref{tab:overall}. Both scalarized and envelope methods can be applied to cases having large state space and reasonably many objectives. However, the size of optimal policy set may affect the performance of these algorithms.
Figure~\ref{fig:sample-ft} shows CR and AE results in three FTN environments with $d=5$ (with 32 solutions), $d=6$ (with 64 solutions), and $d=7$ (with 128 solutions). We observe that both scalarized and envelope algorithms are close to optimal when $d=5$ but both CR and AE values are worse for $d=7$. However, the envelope version is more stable and outperforms the scalarized MORL algorithm in all three cases. These results point to the robustness and scalability of our algorithms.

\paragraph{Sample Efficiency.} 
To compare sample efficiency during the learning phase, we train both our scalarized and envelope deep MORL on the FTN task with different depths for 5,000 episodes. We compute coverage ratio (CR) over 2,000 episodes and adaptation error (AE) over 5,000 episodes. Figure~\ref{fig:sample-ft} shows plots for the metrics computed over a varying number of sampled preferences $N_{\omega}$ (more details can be found in the supplementary material). Each point on the curve is averaged over 5 experiments. We observe that the envelope MORL algorithm consistently has a better CR and AE scores than the scalarized version, with smaller variances. As $N_{\omega}$ increases, CR increases and AE decreases, which shows better use of historical interactions for both algorithms when $N_{\omega}$ is larger. And to achieve the same level AE the envelope algorithm requires smaller $N_\omega$ than the scalarized algorithm. This reinforces our theoretical analysis that the envelope MORL algorithm has better sample efficiency than the scalarized version.

% \begin{figure}[t]
\begin{wrapfigure}{R}{0.45\textwidth}
    \centering
    \vspace{-2em}
    \includegraphics[width=0.45\columnwidth]{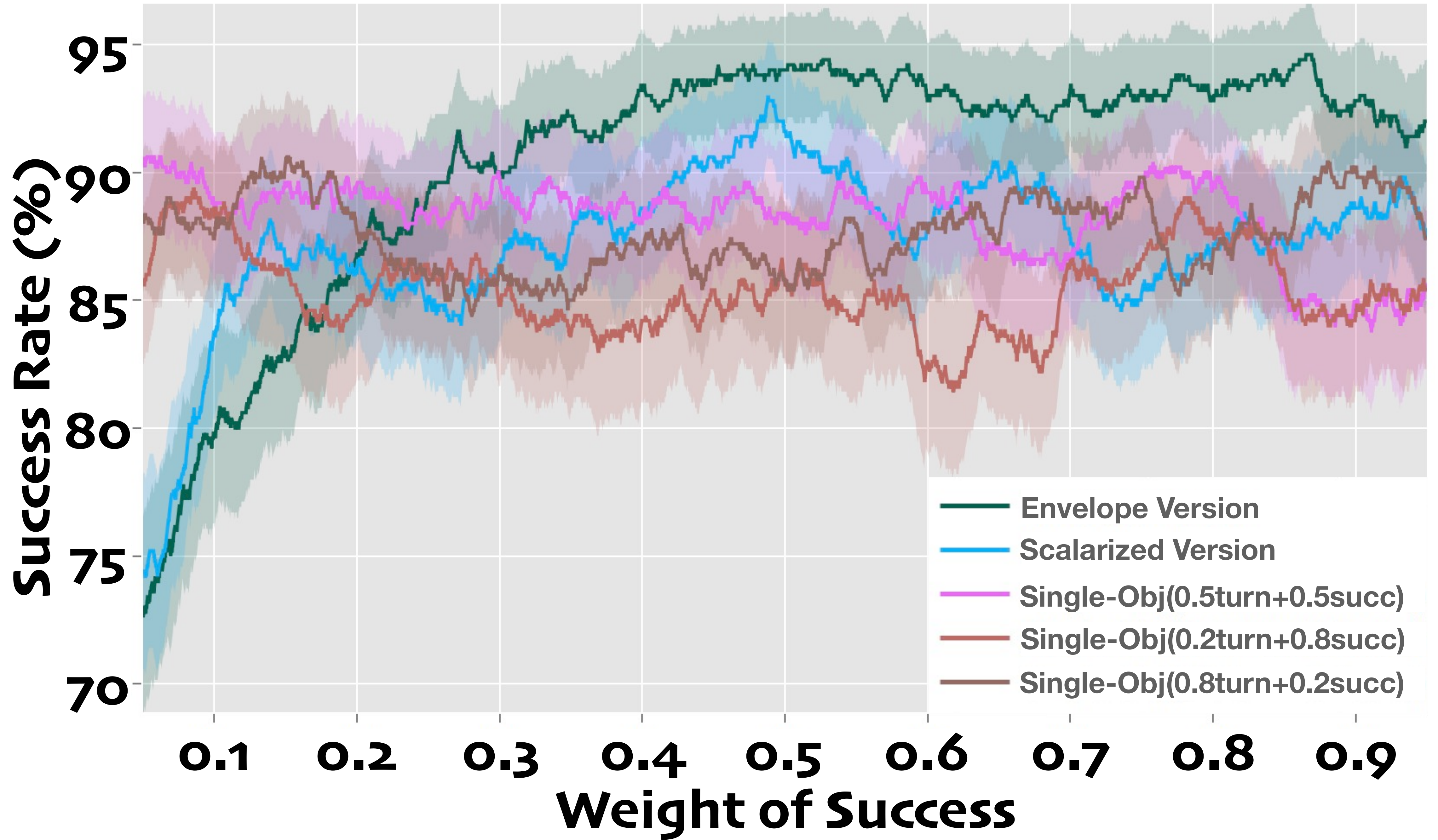}
    \vspace{-1.0em}
    \caption{\small The success-weight curves of task-oriented dialog. Each data point is a moving average of closest around 500 dialogues in the interval of around $\pm$ 0.05 weight of success. The light shadow indicates the standard deviations of 5 independent runs under each configuration.}
    \vspace{-1.0em}
    \label{fig:dialogue-dist-succ}
% \end{figure}
\end{wrapfigure} 

\paragraph{Policy Adaptation.} 

We show how the MORL agents respond to user preference during the adaptation phase in the dialog policy learning task, where the agent must trade off between the dialog success rate and the conversation brevity. Figure~\ref{fig:dialogue-dist-succ} shows the success rate (SR) curves as we vary the weight of the preference on task completion success. The success rates of both MORL algorithms increase as the user's weight on success increases, while those of the single-objective algorithms do not change. This shows that our envelope MORL agent can adapt gracefully to the user's preference. Furthermore, our envelope deep MORL algorithm outperforms other algorithms whenever success is relatively more important to the user (weight $> 0.5$).

\paragraph{Revealing underlying preferences.} 
% \runzhe{I changed the paragraph title to ``revealing underlying preferences" because we do not show the final utility increase here?}
%KN: Sounds good!
Finally, we test the ability of our agent to infer and adapt to unknown preferences on FTN and SuperMario. During the learning phase, the agent does not know the underlying preference, and hence learns a multi-objective policy. During the adaptation phase, our agent performs recovers underlying preferences (as described in Section~\ref{s:inference}) to uncover the underlying preference that maximizes utility.
Table~\ref{tab:ftn-enve-infer} shows the learned preferences for 6 different FTN tasks ($\tt v1$ to $\tt v6$) with unknown one-hot preferences $[1,0,0,0,0,0]$ to $[0,0,0,0,0,1]$, respectively, meaning the agent should only care about one elementary nutrition. These were learned in a few-shot adaption setting, using just 15 episodes.
For the SuperMario Game, we implement an A3C~\cite{MnihBMGLHSK16} variant of our envelope MORL agent (see supplementary material for details). Table~\ref{tab:mario-enve-infer} shows the learned preferences for 5 different tasks ($\tt g1$ to $\tt g5$) with unknown one-hot preferences using just 100 episodes.

We observe that the learned preferences are concentrated on the diagonal, indicating good alignment with the actual underlying preferences. For example, in the SuperMario game variant $\tt g4$, the envelope MORL agent finds the preference with the highest weight ($0.6960$) on the coin objective can best describe the goal of $\tt g4$, which is to collect as many coins as possible. We also tested policy adaptation on the original Mario game using game scores for the scalar rewards. We find that the agent learns preference weights of 0.37 for {\tt x-pos} and 0.23 for {\tt time}, which seems consistent with a common strategy that humans employ -- simply move Mario towards the flag as quickly as possible.

\begin{table}[t]
\parbox{.45\linewidth}{
	\centering
	
	\resizebox{0.48\columnwidth}{!}{
	\begin{tabular}{c|c|c|c|c|c|c}
	& {\tt Protein} & {\tt Carbs} & {\tt Fats} & {\tt Vitamins} & {\tt Minerals} & {\tt Water} \\ \hline
  {\tt v1}
  &\cellcolor{pink!96}\underline{0.9639} 
  &\cellcolor{pink!0}0.
  &\cellcolor{pink!07}0.0361
  &\cellcolor{pink!0}0.
  &\cellcolor{pink!0}0.
  &\cellcolor{pink!0}0.\\ \hline
  {\tt v2}
  &\cellcolor{pink!0}0.
  &\cellcolor{pink!91}\underline{0.9067} &\cellcolor{pink!0}0.
  &\cellcolor{pink!0}0.
  &\cellcolor{pink!0}0.
  &\cellcolor{pink!09}0.0933\\ \hline
  {\tt v3}
  &\cellcolor{pink!05}0.0539
  &\cellcolor{pink!0}0. 
  &\cellcolor{pink!94}\underline{0.9461} 
  &\cellcolor{pink!0}0.
  &\cellcolor{pink!0}0.
  &\cellcolor{pink!0}0.\\ \hline
  {\tt v4}
  &\cellcolor{pink!14}0.1366
  &\cellcolor{pink!04}0.0459
  &\cellcolor{pink!00}0. &\cellcolor{pink!75}\underline{0.7503} &\cellcolor{pink!06}0.0671
  &\cellcolor{pink!00}0.\\ \hline
  {\tt v5}
  &\cellcolor{pink!00}0.
  &\cellcolor{pink!01}0.0148
  &\cellcolor{pink!02}0.0291
  &\cellcolor{pink!04}0.0428 &\cellcolor{pink!75}\underline{0.7503}
  &\cellcolor{pink!17}0.1629\\ \hline
  {\tt v6}
  &\cellcolor{pink!00}0.
  &\cellcolor{pink!05}0.0505
  &\cellcolor{pink!00}0.
  &\cellcolor{pink!00}0.
  &\cellcolor{pink!00}0.
  &\cellcolor{pink!95}\underline{0.9495}\\ \hline
	\end{tabular}}\vspace{5pt}
	\caption{\small Inferred preferences of the envelope MOQ-learning algorithm on different FTN ($d=6$) tasks (\texttt{v1} to \texttt{v6}) after only 15 episodes interaction. The underlying preferences are all ones on the diagonal of the table and zeros for the off-diagonal.}
	\label{tab:ftn-enve-infer}
}
~~~~~~
\parbox{.45\linewidth}{
    \centering
	
	\resizebox{0.43\columnwidth}{!}{
	\begin{tabular}{c|c|c|c|c|c}
	& {\tt x-pos} & {\tt time} & {\tt life} & {\tt coin} & {\tt enemy} \\ \hline
  {\tt g1}
  &\cellcolor{pink!53}\underline{0.5288} &\cellcolor{pink!18}0.1770 &\cellcolor{pink!15}0.1500 &\cellcolor{pink!04}0.0470 &\cellcolor{pink!09}0.0972\\ \hline
  {\tt g2}
  &\cellcolor{pink!19}0.1985
  &\cellcolor{pink!22}\underline{0.2237} &\cellcolor{pink!24}0.2485 
  &\cellcolor{pink!14}0.1422 
  &\cellcolor{pink!18}0.1868\\ \hline
  {\tt g3}
  &\cellcolor{pink!22}0.2196
  &\cellcolor{pink!13}0.1296 &\cellcolor{pink!35}\underline{0.3541} &\cellcolor{pink!18}0.1792
  &\cellcolor{pink!12}0.1175\\ \hline
  {\tt g4}&\cellcolor{pink!02}0.0211 &\cellcolor{pink!24}0.2404 &\cellcolor{pink!02}0.0211 &\cellcolor{pink!69}\underline{0.6960} &\cellcolor{pink!02}0.0211\\ \hline
  {\tt g5}&\cellcolor{pink!07}0.0715&\cellcolor{pink!10}0.1038 &\cellcolor{pink!20}0.2069 &\cellcolor{pink!39}0.3922 &\cellcolor{pink!22}\underline{0.2253}\\ \hline
	\end{tabular}}\vspace{5pt}
	\caption{\small Inferred preferences of the envelope multi-objective A3C algorithm in different Mario Game variants (\texttt{g1} to \texttt{g5}) with 100 episodes. The underlying preferences are all ones on the diagonal of the table and zeros for the off-diagonal.}
	\label{tab:mario-enve-infer}
}
    \vspace{-10pt}
\end{table}
\vspace{-0.5em}
\section{Conclusion}
\vspace{-0.5em}
\label{s:conclusion}
We have introduced a new algorithm for multi-objective reinforcement learning (MORL) with linear preferences,  with the goal of enabling few-shot adaptation of autonomous agents to new scenarios. Specifically, we propose a multi-objective version of the Bellman optimality operator, and utilize it to learn a single parametric representation for all optimal policies over the space of preferences. We provide convergence proofs for our multi-objective algorithm and also demonstrate how to use our model to adapt and elicit an unknown preference on a new task. Our experiments across four different domains demonstrate that our algorithms exhibit effective generalization and policy adaptation.

% \runzhe{Do we need acknowledgements?} \karthik{yes, please add acknowledgements. You can thank the anonymous reviewers and people who helped proofread the paper before submission. This section does not need to be within page limit usually (but check)}

\vspace{-0.5em}
\section*{Acknowledgements}
\vspace{-0.5em}
The authors would like to thank Yexiang Xue at Purdue University, Carla Gomes at Cornell University for helpful discussions on multi-objective optimization, Lu Chen, Kai Yu at Shanghai Jiao Tong University for discussing dialogue applications, Haoran Cai at MIT for helping running a part of synthetic experiments, and anonymous reviewers for constructive suggestions.

\small
\bibliography{ms.bib}
\bibliographystyle{unsrt}

\newpage
\appendix
% \onecolumn

%# -*- coding: utf-8-unix -*-
\hrule\hrule\hrule\hrule
\hrule\hrule\hrule\hrule
~
\begin{center}
{\Large\bf Supplementary Material for Generalized Algorithm for Multi-Objective RL and Policy Adaptation}
\end{center}
~
\hrule\hrule\hrule

\section{Theoretical Framework for Value-Based MORL Algorithms}
\label{app:theory}

In this section, we introduce a theoretical framework for analyzing and designing the value-based multi-objective reinforcement learning algorithms. This framework is based on the well-known Banach's Fixed-Point Theorem, which guarantees the existence and uniqueness of fixed-point of a {\em contraction} on a {\em complete metric space}. Therefore, generalizing this theorem a bit, we can imagine all value functions of reinforcement learning are in some metric space, and finding the optimal value or policy is to find the fixed point of a certain contraction on that space. We first recall the following concepts.

\subsection{General Framework for Value-Based Reinforcement Learning}
\label{app:sec:framework}

\begin{defn}[Contraction] Let $(X, d)$ be a metric space. We say that $\mathcal T$ is a {\em contraction}, if there is a real number $\gamma\in[0,1)$ such that 
\begin{equation}
    d(\mathcal T(x), \mathcal T(x')) \leq \gamma d(x,x')
\end{equation}
for all points $x, x' \in X$, where $\gamma$ is called a Lipschitz coefficient for the contraction $\mathcal T$.
\end{defn}

\begin{thm}[Banach's Fixed-Point Theorem]
    Let $(X, d)$ be a complete metric space and let $\mathcal T: X \rightarrow X$ be a contraction. Then there exists a {\bf \em unique} fixed point $x^{*} \in X$ such that $\mathcal T(x^*) = x^*$.
    Moreover, if $x$ is any point in $X$ and $\mathcal T^{n}(x)$ is inductively defined by $\mathcal T^{n}(x) = \mathcal T(\mathcal T^{n-1}(x))$, then we have $\mathcal T^{n}(x)\rightarrow x^{*}$ as $n\rightarrow \infty$.
    \label{thm:contraction}
\end{thm}

The above introduced Banach fixed-point theorem is well-known.
Readers may refer to the book~\cite{khamsi2011introduction} for more details.
% Readers who are unfamiliar with this theorem and interested in the proof, may refer to the book~\cite{khamsi2011introduction}.
Practically, this provides us with an iterative method for converging to any desired solution in the large solution space, by repeatedly applying a properly designed contraction. For example, the foundation for standard value-based single-objective reinforcement learning is the use of Bellman's optimality equation~\cite{Bellman1957}: 
\begin{equation}
    Q^*(s,a) = r(s,a) + \gamma \mathbb E_{s'\sim\mathcal P(\cdot|s,a)} \sup_{a'\in \mathcal A} Q^*(s',a'),    
    \label{eq:std-opt}
\end{equation}
where $\gamma\in [0,1)$ is the discount factor and the optimal Q-value function $Q^*(s, a)$ is the desired solution in the space $\mathcal Q \subseteq \mathbb R^{\mathcal S \times \mathcal A}$ consisting of all the bounded functions with $\ell_{\infty}$-distance metric 
\begin{equation}
    d(Q, Q') := \sup_{s \in \mathcal S, a\in \mathcal A}|Q(s, a) - Q'(s,a)|.
\end{equation}
Since the all the functions in this space is bounded, it follows that with this $\ell_{\infty}$-distance metric, the space $(\mathcal Q, d)$ is complete. Besides, according to the equation (\ref{eq:std-opt}), we can design an {\em Bellman optimality operator} $\mathcal T$ such that
\begin{equation}
    (\mathcal T Q)(s,a):= r(s,a)+\gamma\mathbb E_{s' \sim \mathcal P(\cdot|s, a)} \sup_{a'\in \mathcal A} Q(s',a'),
\end{equation}
which can be shown as a contraction on $(\mathcal Q, d)$. Many popular value-based reinforcement algorithms, such as {\em deep Q-learning}~\cite{MnihKSRVBGRFOPB15}, can be seen as asynchronous iteration methods with approximately applied contraction.

We can verify that the Bellman optimality operator $\mathcal T$ indeed is a contraction on $(\mathcal Q, d)$, and the optimal value function $Q^*$ is a fixed point in $(\mathcal Q, d)$. Therefore, we can find the unique optimal Q-value function by applying the optimality operator iteratively many times on any initial Q-value function. Similarly, we can also define {\em Bellman evaluation operator} $\mathcal T_{\pi}$ using the Bellman expectation equation $(\mathcal T_{\pi} Q)(s,a) := r(s,a)+\gamma\mathbb E_{\tau\sim (\mathcal P, \pi)}Q(s',a')$, which is also a contraction.

Knowing that the optimality operator is a contraction is important. In practice, we use a minibatch to update previous Q-value function approximated by neural networks, not updating all states and actions. Thus the updated Q-value function is not a strict $\mathcal T Q$, but only close to $\mathcal T Q$ on some state and action pairs. We can still provide a theoretical guarantee that a minibatch iterative algorithm can still converge to a promising result, under certain extra assumptions.

\begin{defn}[Minibatch Iteration]
    Consider the Q-value function $Q$ as a composition of $\{Q_{S,A}\}_{S\subseteq\mathcal S, A \subseteq \mathcal A}$ such that in each iteration,
    \begin{equation*}
        Q^{k+1}_{S,A}(s, a) := \begin{cases}
            \mathcal T Q^{k}_{S,A}(s, a), & \textrm{if~} s \in S \textrm{~and~} a\in A;\\
             Q^{k}_{S,A}(s, a), & \textrm{otherwise.}
        \end{cases}
    \end{equation*}
\end{defn}
\begin{thm}[Minibatch Convergence Theorem]
     Suppose each restricted Q-value function $Q_{S,A}$ can be update an arbitrary number of times, and there is a nest sequence of nonempty sets $\{\mathcal Q^k\}_{k \in \mathbb{Z}}$ with $\mathcal Q^{k+1}\subseteq \mathcal Q^{k} \subseteq \mathcal Q$, $k = 0, 1, \dots$ such that if $\{Q^k\}_{k \in \mathbb{N}}$ is any sequence with $Q^k\in \mathcal Q^k$ for all $k \geq 0$, then $\{Q^k\}$ converges pointwisely to $Q^*$. Assume further the following:
     \begin{enumerate}
         \item Convergence Condition: We have 
         \begin{equation}
             \forall Q \in \mathcal Q^k, \mathcal T Q \in \mathcal Q^{k+1};
         \end{equation}
        \item Box Condition: For all $k, \mathcal Q^k$ is a Cartesian product of the form
        \begin{equation}
            \mathcal Q^k = \times_{s \in \mathcal S, a \in \mathcal A} \mathcal Q_{\{s\},\{a\}}^k,
        \end{equation}
        where $\mathcal Q_{S,A}^k$ is a set of bounded real-valued functions on states $S$ and actions $A$.
     \end{enumerate}
     Then for every $Q^0 \in \mathcal Q^0$ the sequence $\{Q^k\}$ generated by the minibatch iteration algorithm converges to $Q^* $~\cite{Bertsekas17}. 
     \label{thm:minibatch}
\end{thm}

\begin{proof}
    Showing the convergence of the algorithm is equivalent to showing that the iterations of elements from $\mathcal Q^k$ will get in to $\mathcal Q^{k+1}$ eventually, i.e., for each $k \geq 0$, there is a time $t_k$ such that $Q^t\in \mathcal Q^{k}$ for all $t \geq t_k$. We can prove it by mathematical induction.

When $k = 0$, the statement is true since $Q^0\in \mathcal Q^0$. Assuming the statement is true for a given $k$, we will show there is a time $t_{k+1}$ with the required properties. For each $s\in \mathcal S, a\in \mathcal A$, let set $L_{s,a} = \{t: s^t = s, a^t = a\}$ record the time a minibatch update happens on the state $s$ and action $a$. Let $t(s, a)$ be the first element in $L_{s, a}$ such that $t(s, a) \geq t_{k}$. Then by the convergence condition, we have $\mathcal TQ^{t(s,a)} \in \mathcal Q^{k+1}$ for all $s\in \mathcal S$ and $a \in \mathcal A$. In the view of the box condition, it implies $Q^{t+1}(s,a)\in \mathcal Q_{{s}, {a}}^{k+1}$ for all $t \geq t(s, a)$ for any $s\in \mathcal S$ and $a \in \mathcal A$. Therefore, let $t_{k+1} = \max_{s,a}t(s,a) + 1$, using the box condition, we have $Q^t \in \mathcal Q^{k+1}$ for all $t \leq t_{k+1}$. By mathematical induction, the statement holds, and $\mathcal Q^k$ will shrink in to $\mathcal Q^{k+1}$ eventually. Hence, we have proved $\{Q^t\}$ converges to $Q^* $ finally.

\end{proof}

The above theorem indicates that minibatch update with experience replay will not affect the convergence of iteratively applying the optimality operator $\mathcal T$. This gives us the flexibility to design value-based algorithm's updating scheme. Besides, even though we use deep Q-network to approximate the Q-value function and update it by using gradient descent, this will not impair the magical function of optimality operator too much. This is because if after n-round neural network updates, we always have $\sup_{s\in\mathcal S, a\in \mathcal A} \left|Q^{t+n} - \mathcal T Q^{t}\right| \leq \epsilon$, where $n$ is an arbitrary bounded integer, by applying the triangle inequality we conclude that the final error $d(Q_{final}, Q^*) \leq \epsilon / (1-\gamma)$ is bounded. 

We summarize a general theoretical framework for value-based RL algorithm design, which consists of five elements:
\begin{itemize}[leftmargin=0.2in]
    \item {\bf Value Space} $\mathcal X$: A common choice of $\mathcal X$ is the space of value functions in $\mathbb R^\mathcal S$ or the space of Q-value functions in $\mathbb R^{\mathcal S\times \mathcal A}$. There are many other choices such as the space of ordered pair of $(\mathcal V, \mathcal Q)$ (see 2.6.2 in book~\cite{bertsekas2018abstract} for example), or the space of vectorized value functions as we show in the following sections.
    \item {\bf Value Metric} $d$: It defines the ``distance" between two points in the value space. Besides the basic four requirements, the metric $d$ should ensure a complete metric space $(\mathcal X, d)$ to validate the Banach's fixed-point theorem. A compatible selection of value metric will make the convergence analysis easier.
    \item {\bf Evaluation Operator} $\mathcal T_{\pi}$: We have constructed a recursive expression of a certain value point in the value space associated with some policy, e.g., the Bellman expectation equation, to depict the value of a policy $v_{\pi}$ as a fixed point we desire. Carefully verify that the contraction property holds for $\mathcal T_{\pi}$.
    \item {\bf Optimality Operator} $\mathcal T$: A recursive expression of the optimal point in the value space, e.g., the Bellman optimality equation. Note that when the metric $d$ is the supremum of the absolute value of the difference, and $\mathcal T_{\pi}$ is a contraction, we can prove the contraction property of $\mathcal T$ is always automatically satisfied.
    \item {\bf Updating Scheme}: To make a reinforcement learning algorithm practical and scalable, we need to consider many factors in terms of updating scheme. For example: How do we approximate the value and policies? If it is an online algorithm, how do we trade off exploration and exploitation? All these details will significantly influence the performance of our algorithm on real-world tasks.
\end{itemize}

In summary, there are five essential components for analyzing and designing general value-base reinforcement learning algorithms: {\bf (1) value space}, {\bf (2) value metric}, {\bf (3) evaluation operator}, {\bf (4) optimality operator}, and {\bf (5) updating scheme}. In fact, there is some work~\cite{BellemareDM17} developing distributional reinforcement learning in a way similar to this framework.
We will discuss how to design these five elements of our framework to develop envelope multi-objective reinforcement learning algorithms in the next section.

\subsection{MORL with envelope updates}
\label{app:envelope}
The deep MORL algorithm with scalarized update is capable of solving unknown linear preference scenarios of multi-objective reinforcement learning. However, there are several limitations of this algorithm restrict its applicability and performance in practice. Aiming at solving problems stated in Section \ref{s:related_work}, we design a new algorithm called {envelope deep MORL algorithm}. Following the value-based theoretical framework introduced in Section \ref{app:sec:framework}, our key idea to upgrade the scalarized algorithm is to consider a different value space, where every Q-value function is a mapping to multi-objective solutions, not utilities, and therefore maintains the necessary information for prediction in the adaptation phase. Furthermore, we generalize the optimality filter $\mathcal H$ to use that information to boost up the alignment in the learning phase. 

We consider a new value space $\mathcal Q \subseteq (\Omega \rightarrow \mathbb R^m)^{\mathcal S \times \mathcal A}$, containing all bounded functions $\bm Q(s,a,\omega)$ returning the estimates of preferred expected total rewards under preference $\bm \omega$, which are $m$-dimensional vectors. Besides, we employ a value metric $d$ defined by
\begin{equation}
    d(\bm Q, \bm Q') := \sup_{\substack{s\in\mathcal S, a\in\mathcal A\\ \bm \omega\in\Omega}} |\bm \omega^{\intercal} (\bm Q(s, a, \bm \omega) - \bm Q'(s, a, \bm \omega))|.
\end{equation}
Notice that this metric is a pseudo-metric, since the identity of indiscernibles does not hold for it. It is easy to show that metric space $(\mathcal Q, d)$ is complete.

We refer the Q-value functions in this $\mathcal Q$ {\em multi-objective value functions}. Similar to the scalarized one, we design an evaluation operator and an optimality operator for this envelope version algorithm. As for the updating scheme, we use hindsight experience replay and a homotopy optimization trick.

\subsubsection{Multi-Objective Bellman Optimality Operator}
\label{app:sec:envelope:operator}
In this section, we give the evaluation operator and the optimality operator in the new envelope version value space $(\mathcal Q, d)$ as stated above. The evaluation operator now is even simpler than that of the scalarized version. Give a policy $\pi$, the evaluation operator is defined by 
\begin{equation}
    (\mathcal T_{\pi}\bm Q)(s,a, \bm \omega) := \bm r(s,a) + \gamma \mathbb E_{\tau \sim (\mathcal P, \pi)} \bm Q(s', a', \bm \omega).
\end{equation}
Since now the multi-objective Q-value function is also in a vector form, this evaluation operator is almost the same as that of the single-objective reinforcement learning. It can be easily verified as a contraction.

As for the envelope version of optimality operator, we employ a stronger optimality filter $\mathcal H$, defined by $(\mathcal H \bm Q)(s, \bm \omega) := \arg_{Q}\sup_{a'\in\mathcal A, \bm \omega' \in \Omega} \bm \omega^{\intercal}\bm Q(s, a', \bm \omega')$, where the $\arg_{Q}$ takes the multi-objective value corresponding to the supremum, i.e., $\bm Q(s, a'', \bm \omega'')$ such that $(a'', \bm \omega'') \in \arg\sup_{a'\in\mathcal A, \bm \omega' \in \Omega} \bm \omega^{\intercal}\bm Q(s, a', \bm \omega')$. The return of $\arg_{Q}$ depends on scalarization weights $\bm \omega$, and we use $\arg_{Q}$ for simplicity of notation. This filter is solving the convex envelope of the current Pareto frontier, therefore we name this algorithm as ``envelope" version. We can write the optimality operator $\mathcal T$ in terms of the optimal filter:
\begin{equation}
    (\mathcal T \bm Q)(s, a, \bm \omega) := \bm r(s, a) + \gamma \mathbb E_{s'\sim\mathcal P(\cdot | s,a)}(\mathcal H \bm Q)(s', \bm \omega).
\end{equation}
\begin{repthm}{thm:fix-envelope}[Fixed Point of Evelope Optimality Operator for MORL]
    Use above definitions in the envelope version value space. Let $\bm Q^* \in \mathcal Q$ be the preferred optimal value function in the value space, such that 
    \begin{equation}
        \bm Q^*(s, a, \bm \omega) = \arg_{Q} \sup_{\pi\in \Pi}\bm \omega^{\intercal}\mathbb E_{\tau\sim (\mathcal P, \pi) | s_0=s, a_0=a}\left[\sum_{t=0}^{\infty}\gamma^t \bm r(s_t,a_t) \right],
    \end{equation} 
    where the $\arg_{Q}$ takes the multi-objective value corresponding to the supremum, then $\bm Q^* = \mathcal T \bm Q^*$ holds.
\end{repthm}
\begin{proof}
    First, we observe that $d(\bm Q^*, \mathcal T\bm Q^*)= \sup_{\substack{s\in\mathcal S, a\in\mathcal A\\ \bm\omega\in\Omega}}|\bm \omega^{\intercal} (\bm Q^*(s,a,\bm\omega) - \mathcal T\bm Q^*(s,a,\bm \omega))| = 0 \Leftrightarrow \bm \omega^{\intercal}\mathcal T\bm Q^*(s,a,\bm \omega) = \bm \omega^{\intercal} \bm Q^*(s,a,\bm \omega)$ for all $s, a, \bm \omega$. Then, by substituting the definition of $Q^*$ into $\mathcal T Q^*$,
    % \begin{eqnarray*}
    %     \bm \omega^{\intercal} \mathcal T Q^*(s,a, \bm \omega) & = & \bm \omega^{\intercal} \bm r(s, a) + \gamma \cdot \bm \omega^{\intercal} \mathbb E_{s' \sim \mathcal P(\cdot |s, a)} (\mathcal H Q)(s', \bm \omega)\\
    %     & = & \bm \omega^{\intercal} \bm r(s, a) + \gamma \cdot \bm \omega^{\intercal} \mathbb E_{s'\sim \mathcal P(\cdot|s,a)}\arg_{Q}\sup_{a'\in\mathcal A, \bm \omega' \in \Omega} \bm \omega^{\intercal}\bm Q(s, a', \bm \omega')\\
    %     & = & \bm \omega^{\intercal} \bm r(s, a) + \gamma \sup_{\pi} \mathbb E_{\substack{\tau \sim (\mathcal P, \pi)\\s_0\sim \mathcal P(\cdot| s,a)}} \left[\sum_{t = 0}^{\infty} {\gamma^{t}\bm \omega^{\intercal} \bm r(s_t, a_t)}\right].
    %     \end{eqnarray*}
    %     The last step is because the elimination of $\bm \omega^{\intercal}$ and $\arg_{Q}$, and the expectation over $s'$ does not affect the supremum over $a'$, and the supremum over $a'$ is absorbed by the supremum over $\pi$. Then by the definition of $Q^*$ again, we have
    %     \begin{equation*}
    %     \bm \omega^{\intercal} \mathcal T \bm Q^*(s,a, \bm \omega) = \sup_{\pi\in\Pi} \mathbb E_{\substack{\tau \sim (\mathcal P, \pi) | s_0 = s, a_0 = a}} [\sum_{t = 0}^{\infty}{\gamma^{t}\omega^T r(s_t, a_t)}]\\
    %      = \bm \omega^{\intercal} \bm Q^*(s, a, \bm \omega).
    % \end{equation*}
    
{\tiny
\begin{eqnarray*}
        {\color{cyan}\bm \omega^{\intercal}} \mathcal T \bm Q^*(s,a, {\color{cyan}\bm \omega}) 
        & = & {\color{cyan}\bm \omega^{\intercal}} \bm r(s, a) + \gamma \cdot {\color{cyan}\bm \omega^{\intercal}} \mathbb E_{s' \sim \mathcal P(\cdot |s, a)} (\mathcal H \bm Q^*)(s', {\color{cyan}\bm \omega})\\
      {\color{gray} (\text{def. of } \mathcal H)} 
       & = & {\color{cyan}\bm \omega^{\intercal}} \bm r(s, a) + \gamma \cdot {\color{cyan}\bm \omega^{\intercal}} \mathbb E_{s'\sim \mathcal P(\cdot|s,a)}\arg_{Q}\sup_{a'\in\mathcal A, {\color{orange} \bm \omega' \in \Omega}} {\color{cyan}\bm \omega^{\intercal}}\bm Q^*(s', a', {\color{orange}\bm \omega'})\\
       {\color{gray} (\text{linearity of exp. \& cancel } \bm \omega^{\intercal} \text{ and } \arg_Q)} 
       & = & {\color{cyan}\bm \omega^{\intercal}} \bm r(s, a) + \gamma \cdot \mathbb E_{s'\sim \mathcal P(\cdot|s,a)}\sup_{a'\in\mathcal A, {\color{orange} \bm \omega' \in \Omega}} {\color{cyan}\bm \omega^{\intercal}}\bm Q^*(s', a', {\color{orange}\bm \omega'})\\
       {\color{gray} (\text{insert eq. (20), def. of $\bm Q^*$})} 
       & = & {\color{cyan}\bm \omega^{\intercal}} \bm r(s, a) + \gamma \cdot \mathbb E_{s'\sim \mathcal P(\cdot|s,a)}\sup_{a'\in\mathcal A, {\color{orange} \bm \omega' \in \Omega}} {\color{cyan}\bm \omega^{\intercal}}\left\{\arg_{Q} \sup_{\pi\in \Pi}{\color{orange}\bm \omega'^{\intercal}}\mathbb E_{\substack{\tau\sim (\mathcal P, \pi) \\| s_0=s', a_0=a'}}\left[\sum_{t=0}^{\infty}\gamma^t \bm r(s_t,a_t) \right]\right\}\\
       {\color{gray} (\text{use def. of $\arg_{Q}$, explained below})} 
       & = & {\color{cyan}\bm \omega^{\intercal}} \bm r(s, a) + \gamma \cdot \mathbb E_{s'\sim \mathcal P(\cdot|s,a)}\sup_{a'\in\mathcal A} {\color{cyan}\bm \omega^{\intercal}}\left\{\arg_{Q} \sup_{\pi\in \Pi}{\color{cyan}\bm \omega^{\intercal}}\mathbb E_{\substack{\tau\sim (\mathcal P, \pi) \\| s_0=s', a_0=a'}}\left[\sum_{t=0}^{\infty}\gamma^t \bm r(s_t,a_t) \right]\right\}\\
       {\color{gray} (\text{rearrange expectation and sup})}
        & = & {\color{cyan}\bm \omega^{\intercal}} \bm r(s, a) + \gamma \cdot {\color{cyan}\bm \omega^{\intercal}} \arg_{Q}  \sup_{\pi\in \Pi}{\color{cyan}\bm \omega^{\intercal}}\mathbb E_{\substack{\tau\sim (\mathcal P, \pi) \\ s_0\sim \mathcal P(\cdot | s,a)}}\left[\sum_{t=0}^{\infty}\gamma^t \bm r(s_t,a_t) \right]\\
       {\color{gray} (\text{merge 1st term to sum \& use def. of $\bm Q^*$ again})}
        & = & {\color{cyan}\bm \omega^{\intercal}}\left\{\arg_{Q} \sup_{\pi\in \Pi}{\color{cyan}\bm \omega^{\intercal}}\mathbb E_{\substack{\tau\sim (\mathcal P, \pi) \\| s_0=s, a_0=a}}\left[\sum_{t=0}^{\infty}\gamma^t \bm r(s_t,a_t) \right] \right\} = {\color{cyan}\bm \omega^{\intercal}} \bm Q^*(s,a, {\color{cyan}\bm \omega})
\end{eqnarray*}}

The fourth equation is due to a sandwich inequality, 
{\scriptsize $\displaystyle {\color{cyan}\bm \omega^{\intercal}} \arg_{Q}  \sup_{\pi\in \Pi}{\color{cyan}\bm \omega^{\intercal}}\bm Q^{\pi} \leq \sup_{{\color{orange} \bm \omega' \in \Omega}}{\color{cyan}\bm \omega^{\intercal}} \arg_{Q}  \sup_{\pi\in \Pi}{\color{orange}\bm \omega'^{\intercal}}\bm Q^{\pi} = {\color{cyan}\bm \omega^{\intercal}} \arg_{Q}  \sup_{\pi\in \Pi}{\color{orange}\bm \omega_*^{\intercal}}\bm Q^{\pi} = {\color{cyan}\bm \omega^{\intercal}} \bm Q^{{\color{orange} \pi_{\bm \omega_*'}'}} \leq {\color{cyan}\bm \omega^{\intercal}} \arg_{Q}  \sup_{\pi\in \Pi}{\color{cyan}\bm \omega^{\intercal}}\bm Q^{\pi}$}, 
where ${\color{orange} \bm \omega_{*}'}$ and ${\color{orange} \pi_{\bm \omega_*'}'}$ are preference and policy corresponding to the supremums. According to the observation stated at the beginning, $d(\bm Q^*, \mathcal T \bm Q^*) = 0$. The preferred optimal value function is a fixed point of the proposed envelope version optimality operator.
\end{proof}

Theorem \ref{thm:fix-envelope} tells us the preferred optimal value function is one of the fixed-points of envelope optimality operator $\mathcal T$ in the value space. And we still need to show that this $\mathcal T$ is a contraction. 
\begin{repthm}{thm:contraction-envelope}[Envelope Optimal Operator is a Contraction]
    Let $\bm Q, \bm Q'$ be any two multi-objective Q-value functions in the envelope value space $\mathcal Q$ as defined above, the Lipschitz condition $d(\mathcal T \bm Q, \mathcal T \bm Q') \leq \gamma d(\bm Q, \bm Q')$ holds, where $\gamma \in [0, 1)$ is the discount factor of the underlying MOMDP $\mathcal M$ (see Section \ref{sec:momdp}).
\end{repthm}
\begin{proof}
    Without the loss of generality, we assume $\sup_{a\in \mathcal A, \bm \omega'\in\Omega}\bm \omega^{\intercal}\bm Q(s, a, \bm \omega') \geq \sup_{a\in \mathcal A, \bm \omega'\in\Omega}\bm \omega^{\intercal} \bm Q'(s, a, \bm \omega')$ for some state $s$ and $\bm \omega$ of interest. Expand the expression of $d(\mathcal T\bm Q, \mathcal T\bm Q')$ we have
    
    \begin{eqnarray*}
         d(\mathcal T \bm Q, \mathcal T \bm Q')(s,a) & = & \sup_{\substack{s\in\mathcal S, a\in\mathcal A\\ \bm \omega\in\Omega}}\left| \bm \omega^{\intercal}((\mathcal T \bm Q)(s,a) - (\mathcal T \bm Q')(s,a))\right|\\
        & = & \sup_{\substack{s\in\mathcal S, a\in\mathcal A\\ \bm \omega\in\Omega}}\left| \gamma \cdot \bm \omega^{\intercal} \mathbb E_{s' \sim P(\cdot |s, a)} (\mathcal H \bm Q)(s',\bm \omega) - \gamma \cdot \bm \omega^{\intercal}\mathbb E_{s' \sim P(\cdot |s, a)} (\mathcal H \bm Q')(s', \bm \omega)\right|\\
        & \leq & \gamma\cdot\sup_{\substack{s'\in\mathcal S, \bm \omega\in\Omega}}\scalebox{0.88}{$\displaystyle \left|\bm \omega^{\intercal} \left[\arg_{Q}\sup_{a'\in\mathcal A, \bm \omega' \in \Omega} \bm \omega^{\intercal}\bm Q(s', a', \bm \omega')  - \arg_{Q}\sup_{a''\in\mathcal A, \bm \omega'' \in \Omega} \bm \omega^{\intercal} Q'(s', a'', \bm \omega'') \right]\right|$}\\
        & \leq & \gamma\cdot\sup_{\substack{s'\in\mathcal S, \bm \omega\in\Omega}} \left|\sup_{a'\in\mathcal A, \bm \omega' \in \Omega} \bm \omega^{\intercal}\bm Q(s', a', \bm \omega')  - \sup_{a''\in\mathcal A, \bm \omega'' \in \Omega} \bm \omega^{\intercal} Q'(s', a'', \bm \omega'') \right|\\        
    \end{eqnarray*}
    Step 2 to 3 is because $|\mathbb E[\cdot]| \leq \mathbb E[|\cdot|] \leq \sup|\cdot|$, and step 3 to 4 results from the cancellation between ${\color{cyan}\bm \omega^{\intercal}}$ and ${\color{violet} \arg_{Q}}$ (as justified above). 
    According to our assumption, let $a'$ and $\bm \omega'$ be the action and preference chosen to maximize the value of $\bm \omega^{\intercal}\bm Q$ for some state $s'$ and preference $\bm \omega$ of interest, then we derive
    \begin{eqnarray*}
        d(\mathcal T \bm Q, \mathcal T \bm Q')(s,a)    
        & \leq & \gamma\cdot\sup_{s'\in\mathcal S, \bm \omega \in \Omega} \left| \bm \omega^{\intercal} \bm Q(s', a', \bm \omega')  - \sup_{a''\in\mathcal A, \bm \omega'' \in \Omega} \bm \omega^{\intercal} \bm Q'(s', a'', \bm \omega'') \right|\\        
        & = & \gamma\cdot\sup_{s'\in\mathcal S, \bm \omega \in \Omega} | \bm \omega^{\intercal} \bm Q(s', a', \bm \omega')  - \bm \omega^{\intercal} \bm Q'(s', a', \bm \omega') \\
        && + \bm \omega^{\intercal} \bm Q'(s', a', \bm \omega')  -\sup_{a''\in\mathcal A, \bm \omega'' \in \Omega} \bm \omega^{\intercal} \bm Q'(s', a'', \bm \omega'') |\\
        & \leq & \gamma\cdot\sup_{s'\in\mathcal S, \bm \omega \in \Omega} \left| \bm \omega^{\intercal} \bm Q(s', a', \bm \omega')  - \bm \omega^{\intercal} \bm Q'(s', a', \bm \omega') \right|\\
        & \leq & \gamma\cdot\sup_{\substack{s'\in\mathcal S, a'\in\mathcal A\\ \omega\in\Omega}} \left| \bm \omega^{\intercal} \bm Q(s', a', \bm \omega')  - \bm \omega^{\intercal} \bm Q'(s', a', \bm \omega')\right| = \gamma d(\bm Q, \bm Q')
    \end{eqnarray*}
    The step 2 to 3 arises from the w.l.o.g. assumption that {\small $\displaystyle {\color{cyan}\bm \omega^{\intercal}} \bm Q(s', a', {\color{orange}\bm \omega'})  - \sup_{a'',{\color{brown} \bm \omega''}}{\color{cyan}\bm \omega^{\intercal}} \bm Q'(s', a'', {\color{brown}\bm \omega''}) \geq 0$}, as stated in lines 612 and 615. Thus, the whole expression in $|\cdot|$ is nonnegative and
{\small ${\color{cyan}\bm \omega^{\intercal}} \bm Q(s', a', {\color{orange}\bm \omega'})  - {\color{cyan}\bm \omega^{\intercal}} \bm Q'(s', a', {\color{orange}\bm \omega'}) \geq 0$}
. We can discard the last two terms since
{\small $\displaystyle {\color{cyan}\bm \omega^{\intercal}} \bm Q'(s', a', {\color{orange}\bm \omega'})\leq \sup_{a''\in\mathcal A, {\color{brown}  \bm \omega'' \in \Omega}} {\color{cyan}\bm \omega^{\intercal}} \bm Q'(s', a'', {\color{brown} \bm \omega''})$}.
Step 3 to 4 is because 
{\small $\displaystyle \sup_{s', \bm \omega'} f(s',a', \bm \omega') \leq \sup_{s', a'', \bm \omega'} f(s',a'',\bm \omega')$} 
holds for any $a'$ and $f(\cdot)$.
   
This completes our proof that $\mathcal T$ is a contraction.
\end{proof}
Remember that in our design, envelope version value distance $d$ is a pseudo-metric. In a pseudo-metric space iteratively applying contraction may not shrink to the desired fixed point. To assert the convergence effectiveness of our design for optimality operator, we need to investigate a generalized Banach's Fixed-Point Theorem in the pseudo-metric space.

\subsubsection{Generalized Banach's Fixed-Point Theorem}
\label{app:sec:envelope:theory}
\begin{repthm}{thm:generalized-fix}[Generalized Banach Fixed-Point Theorem]
    Given that $\mathcal T$ is a contraction mapping with Lipschitz coefficient $\gamma$ on the complete pseudo-metric space $\langle \mathcal Q, d \rangle$, and $\bm Q^*$ is defined as that in Theorem \ref{thm:fix-envelope}, it is always true that $\lim_{n\rightarrow\infty} d(\mathcal T^n \bm Q, \bm Q^*) = 0$ for any $\bm Q\in \mathcal Q$.
\end{repthm}
\begin{proof}
    By the symmetry and triangle inequality of pseudo-metric, for any $\bm Q, \bm Q' \in \mathcal Q$,
    \begin{eqnarray*}
        d(\bm Q, \bm Q') & \leq & d(\bm Q, \mathcal T\bm Q) + d(\mathcal T\bm Q, \mathcal T\bm Q') + d(\mathcal T\bm Q', \bm Q')\\
        & \leq & d(\bm Q, \mathcal T\bm Q) + \gamma d(\bm Q, \bm Q') + d(\mathcal T\bm Q', \bm Q')\\
        \Rightarrow d(\bm Q, \bm Q') & \leq &  [d(\mathcal T \bm Q, \bm Q)+ d(\mathcal T \bm Q', \bm Q')]/(1-\gamma)
    \end{eqnarray*}
    Consider two points $\mathcal T^\ell \bm Q, \mathcal T^m \bm Q$ in the sequence $\{\mathcal T^n \bm Q\}$. Their distance is bounded by
    \begin{eqnarray*}
        d(\mathcal T^\ell \bm Q, \mathcal T^m \bm Q) & \leq & [d(\mathcal T^{\ell+1} \bm Q, \mathcal T^\ell \bm Q) + d(\mathcal T^{m+1} \bm Q, \mathcal T^m \bm Q)] / (1-\gamma)\\
        & \leq & [\gamma^\ell d(\mathcal T \bm Q, \bm Q) + \gamma^m d(\mathcal T \bm Q, \bm Q)]/(1-\gamma)\\
        & \leq & \frac{\gamma^{\ell} + \gamma^{m}}{(1-\gamma)} d(\mathcal T\bm Q, \bm Q)
    \end{eqnarray*}
    since $\gamma \in [0, 1)$ the distance $d(\mathcal T^\ell \bm Q, \mathcal T^m \bm Q)$ converge to $0$ as $\ell, m \rightarrow \infty$, proving that $\{\mathcal T^n \bm Q\}$ is a Cauchy sequence. Because $\langle \mathcal Q, d \rangle$ is a complete pseudo-metric space, $\lim_{n\rightarrow \infty} d(\mathcal T^n \bm Q, \bm Q^\diamond) = 0$ for some $\bm Q^\diamond \in \mathcal Q$. Therefore,
    \[d(\mathcal T \bm Q^\diamond, \bm Q^\diamond) = \lim_{n\rightarrow\infty}d(\mathcal T^{n+1} \bm Q, \bm Q^\diamond) = \lim_{n\rightarrow\infty}d(\mathcal T^{n} \bm Q, \bm Q^\diamond) = 0\]
    We claim that $\bm Q^\diamond$ and $\bm Q^*$ must lie in the same equivalent class partitioned by relation $d(\cdot, \cdot) = 0$. Suppose $d(\bm Q^\diamond, \bm Q^*) \neq 0$, then we can get a contradiction
    \begin{eqnarray*}
        d(\bm Q^\diamond, \bm Q^*) & = & d(\bm Q^\diamond, \mathcal T \bm Q^\diamond) + d(\mathcal T\bm Q^\diamond, \mathcal T \bm Q^*) + d(\mathcal T\bm Q^*, \bm Q^*)\\
        & \leq & 0 + \gamma d(\bm Q^\diamond, \bm Q^*) + 0\\
        & < & d(\bm Q^\diamond, \bm Q^*)
    \end{eqnarray*}
    This proves our claim. Therefore, $\lim_{n\rightarrow\infty} d(\mathcal T^n \bm Q, \bm Q^*) = 0$ for any $\bm Q\in \mathcal Q$.
\end{proof}

In other words, Theorem \ref{thm:generalized-fix} guarantees that iteratively applying optimal operator $\mathcal T$ on any multi-objective Q-value function, the algorithm will terminate with a function $\bm Q^{\diamond}$ which is equivalent to $\bm Q^*$ under the measurement of pseudo-metric $d$. Actually, these $\bm Q^{\diamond}$'s are as good as $\bm Q^*$, since they all have the same utilities for each $\bm \omega$, and only differ in the real value when the utility corresponds a recess in the utility control frontier.

\subsubsection{Updating Scheme}
\label{app:sec:envelope:update}

\paragraph{Hindsight Experience Replay (HER)}

Paper~\cite{AndrychowiczCRS17} presents a technique for training a reinforcement learning to serve multiple goals. For each episode, the agent performs following a policy according to a randomly sampled goals. When updating, the agent uses the past trajectory update for multiple other goals in parallel. They referred this method as {\em hindsight experience replay} (HER). Though our settings are completely different, a similar method can be employed to update utility-based multi-objective Q-network here.

In the learning phase of the unknown linear preference scenario, for each training episode, the MORL agent randomly sample a preference  $\bm \omega$ from a certain distribution $\mathcal D_{\omega}$. When updating the multi-objective Q-network accordingly, for each sampled transition record $(s_t^i, a_t^i, \bm r_t^i, s_{t+1}^i)$ from replay buffer $\mathcal D_{\tau}$, we associate it with $N_{\omega}$ preferences $\{\bm \omega_{1}, \bm \omega_{2}, \dots, \bm \omega_{N_{\omega}}\}$ sampled from $\mathcal D_{\omega}$. The update is applied to an expended batch of size $\mathtt{minibatch\_size \times N_{\omega}}$. Note that the preferences sampled for actions only influence the agent’s actions, not the environment dynamics. In this way, the trajectories can be replayed with arbitrary preferences with ``hindsight".

\paragraph{Homotopy Optimization}
We use deep neural networks to approximate bounded functions in $\mathcal Q \subseteq (\Omega \rightarrow \mathbb R^m)^{\mathcal S \times \mathcal A}$ with parameters $\theta$. We refer to this neural network as a {\em multi-objective Q-network}. To drag $\bm Q$ close to $\mathcal T \bm Q$ at each update step, the multi-objective Q-network can be trained by minimizing a series of loss functions 
\begin{equation}
    L^{\mathtt A}_{k} (\theta) = \mathbb E_{s, a, \bm \omega} \left[ \left\| \bm y_k - \bm Q(s, a, \bm \omega; \theta) \right\|_2^2 \right],
    \label{eq:enve-lossA-2}
\end{equation}
which changes at each iteration $k$, where $\bm y_{k} = \mathbb E_{s'} \left[ \bm r(s,a) + \gamma (\mathcal H \bm Q)(s', a', \bm \omega; \theta_{k}) \right]$ is the target of iteration $k$. The target is fixed during optimizing this loss function. 

Minimizing the loss function $L^{\mathbb A}$ is trying to drag the vector $\bm Q$ close to $\mathcal T \bm Q$. This ensures the correctness of our algorithm, to predict a $\bm Q$ as the real solution, while this means the square error is hard to be optimized in practice. To address this problem, we use a sequence of auxiliary loss function $L^{\mathbb B}$ to directly optimized the value metric $d$, which is defined by
\begin{equation}
    L^{\mathtt B}_{k}(\theta) = \mathbb E_{s,a,\bm \omega}[\left|\bm \omega^{\intercal}\bm y_k - \bm \omega^{\intercal} \bm Q(s,a,\bm \omega;\theta)\right|]
    \label{eq:enve-lossB-2}
\end{equation}

Our final loss function sequence $L_{k}(\theta) = (1-\lambda_k)\cdot L^{\mathtt A}_{k}(\theta) + \lambda_k \cdot L^{\mathtt B}_{k}(\theta)$, where $\lambda_k$ is a weight to trade off between losses $L^{\mathtt A}_k$ and $L^{\mathtt B}_k$. We increase the value of $\lambda_k$ from $0$ to $1$, to shift our loss function from  $L^{\mathtt A}$ to  $L^{\mathtt B}$. This method called {\em homotopy optimization}~\cite{watson1989modern} is effective since for each update step, it uses the optimization result from the previous step as the initial guess. In the envelope deep MORL algorithm, $L^{\mathtt A}$ first ensure the prediction of $\bm Q$ is close to any real expected total reward, though it is hard to be optimal. And then $L^{\mathtt B}$ can provide an auxiliary force to pull the current guess along the direction with better utility. Figure \ref{fig:homotopy} illustrate an explanation for this homotopy optimization.

The parameters of the multi-objective Q-network will be updated by $\theta_{k+1} \leftarrow \theta_k - \eta$, where 
\begin{equation}
    \eta \propto \nabla_{\theta = \theta_k}L_k(\theta) = -\mathbb E_{s, a, s'}\left[\left(\bm r + \gamma (\mathcal H \bm Q)(s', a', \bm \omega; \theta_{k}) - \bm Q(s, a, \bm \omega; \theta_k)\right)^{\intercal}\nabla_{\theta=\theta_k}\bm Q(s,a, \bm \omega;\theta) \right].
    \label{eq:envelope-dqn-2}
\end{equation}

\begin{figure}[h]
    \centering
    \includegraphics[width=\textwidth]{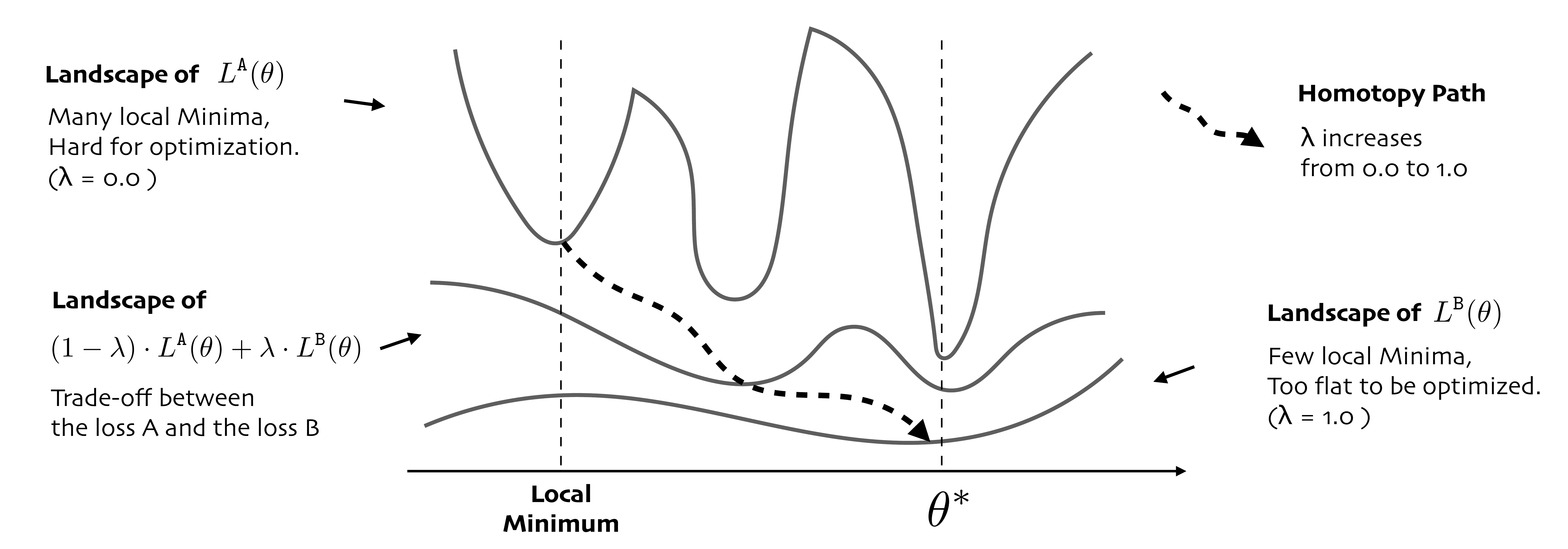}
    \caption{An explanation for homotopy optimization method used in the envelope deep MORL algorithm. The MSE loss $L^{\mathtt A}$ is hard for optimization since there are many local minima over its landscape. Although the value metric loss $L^{\mathtt B}$ has fewer local minima, it is also hard for optimization since there are many vectors $\bm Q$  minimizing value metric $d$. The landscape of $L^{\mathtt B}$ is too flat. The homotopy path connecting $L^{\mathtt A}$ and $L^{\mathtt B}$ provides better opportunities to find the global optimal parameters $\theta^*$}
    \label{fig:homotopy}
\end{figure}

When updating, we sample a minibatch of transition records from this replay buffer with HER. Theorems \ref{thm:fix-envelope}-\ref{thm:generalized-fix} and \ref{thm:minibatch} guarantees the convergence of this minibatch updating, with an extra assumption that we can update the Q-function according to equation~\ref{eq:envelope-dqn-2} for each $\bm \omega \in \Omega$ infinite times. We use {\em hindsight experience replay} (HER) to ensure this. Notice that we will apply our optimality filter on the HER expended batch. Therefore the cost of solving the convex envelope is acceptable. Our multi-objective Q-network can also be replaced with other models similar to those in single-objective off-policy RL algorithms. In the experiment, we also use some popular deep reinforcement learning techniques to stabilize and speed up our algorithms. The skeleton of our envelope deep MORL is shown as Algorithm \ref{algo:envelope}.
\section{Experimental Details}
\label{app:experiments}
We first demonstrate our experimental results on two synthetic domains, Deep Sea Treasure (DST) and Fruit Tree Navigation (FTN), as well as two complex real domains, Task-Oriented Dialog Policy Learning (Dialog) and SuperMario Game (SuperMario). We also elaborate specific model architecture and and implementation details in this section.

\subsection{Domain Details}
    
\begin{wrapfigure}{R}{0.45\textwidth}
    % \vspace{-0.2em}
    \includegraphics[width=0.45\columnwidth]{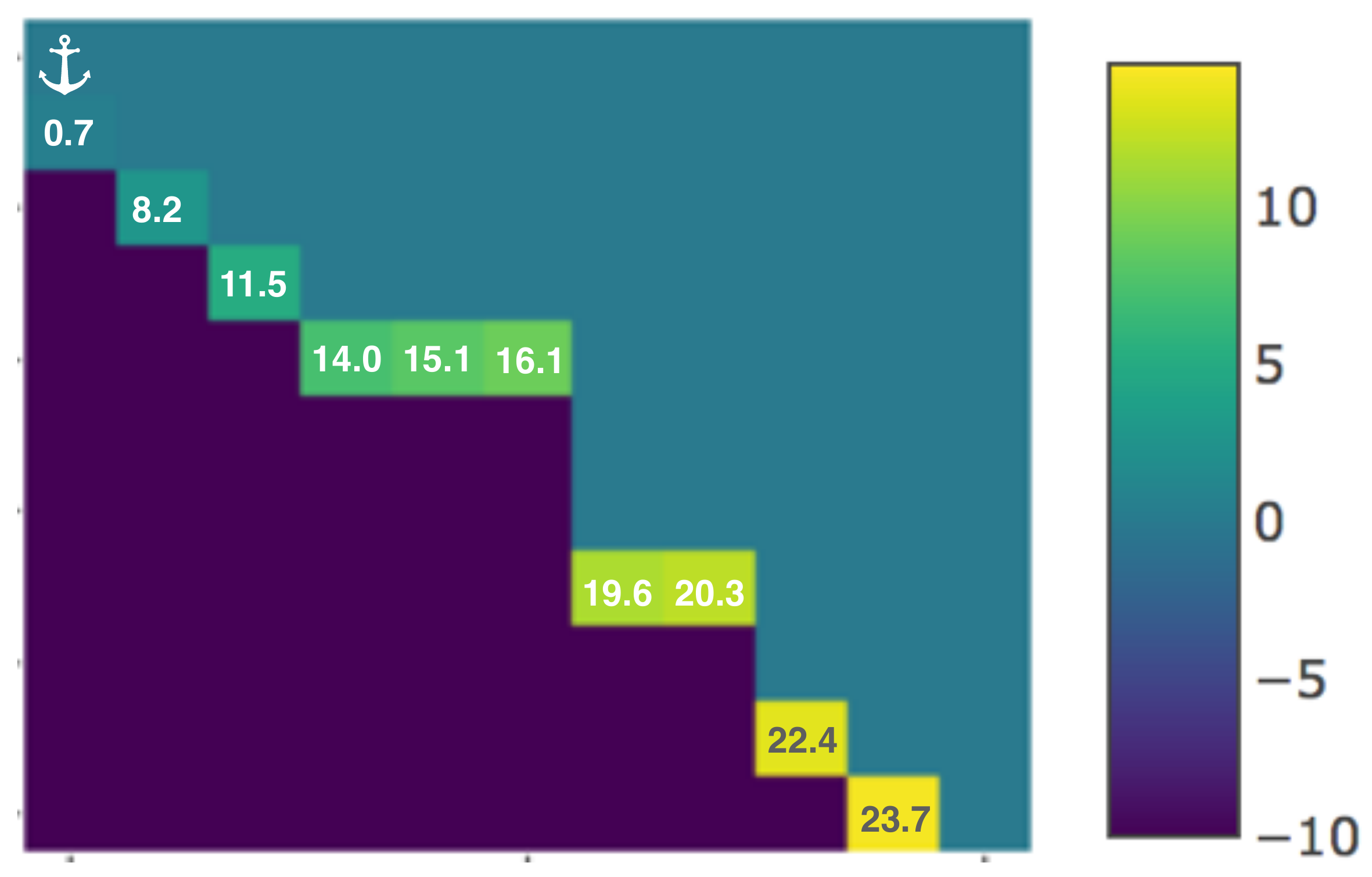}
%    \vspace{-1.5em}
    \caption{Deep Sea Treasure (DST):  An agent controls a submarine searching for treasures in a $10 \times 11$-grid world.  The state $s_t$ consists of the agent’s current coordinates $(x, y)$. An agent’s action spaces is navigation in four directions. The reward received by the agent is a 2-dimensional vector (time penalty, treasure value).}
%   	\vspace{-0.2em}
    \label{fig:dst-env}
\end{wrapfigure}

\paragraph{Deep Sea Treasure (DST)} Our first experiment domain is a grid-world navigation problem, {\em Deep Sea Treasure}. This episodic problem was first explicitly created to highlight the limitations of linear scalarization~\cite{VamplewDBID11} . However, in this paper, we use this environment as a delayed linear preference scenario. We ensure the Pareto frontier of this environment is convex, therefore the Pareto frontier itself is the its CCS.

In DST, an agent controls a submarine searching for treasures in a $10 \times 11$-grid world while trading off {\tt time-cost} and {\tt treasure-value}. The grid world contains 10 treasures of different values. Their values increase as their distances from the starting point $s_0= (0, 0)$ increase. An agent's action spaces are formed by navigation in four directions. The reward has two dimensions: the first dimension indicates a time penalty, which is $-1$ on all turns; and the second dimension is the treasure value which is $0$ except when the agent moves into a treasure location. We ensure the Pareto frontier of this environment to be convex. We depicted the map in Figure~\ref{fig:dst-env}.

\paragraph{Fruit Tree Navigation (FTN)} Our second experiment domain is a full binary tree of depth $d$ with randomly assigned vectorial reward $\bm r\in \mathbb R^6$ on the leaf nodes. These rewards encode the amounts of six different components of nutrition of the {\em fruits} on the tree: $\{\mathtt{Protein}, \mathtt{Carbs}, \mathtt{Fats}, \mathtt{Vitamins}, \mathtt{Minerals}, \mathtt{Water}\}$. For every leaf node, $\exists \bm \omega$ for which its reward is optimal, thus all leaves lie on the $\mathtt{CCS}$. The goal of our MORL agent is to find a path from the root to a leaf node that maximizes utility for a given preference, choosing between left or right subtrees at every non-terminal node.

Figure \ref{fig:ftn-env} shows an instance of the {\em fruit tree navigation} task when $d=6$, in which every non-leaf node is associated with zero reward and every fruit is a potential optimal solution in the convex cover set of the Pareto frontier. To construct this, we sample $\bm r^{(i)} = (\bm v^{(i)}_{+} + \bm v^{(i)}_{-})/ \|\bm v^{(i)}\|_{2}$, where $\bm v^{(i)}\sim \mathcal N_{6}(\bm 0,\bm I)$, for each fruit $i$ on a leaf node. The optimal multiple-policy model for this tree structured MOMDP should contain all the paths from the root to different desired fruits. In experiments, we also test on $d=5$ and $d=7$ cases.

In this multi-objective environment, an optimal policy can be easily learned if we know the preference function $f_{\bm \omega}(\cdot ) = \langle \bm \omega, \cdot \rangle$ for scalarization. However, since here we are interested in evaluating whether a multiple-policy neural network, trained with deep MORL algorithms, can find and maintain all the potential optimal policies (i.e., paths to every leaf node) when the preference function is unknown, and adapt to the optimal policy when a specific preference is given or hidden during execution.

\begin{figure}[t]
    \centering
    \includegraphics[width=\columnwidth]{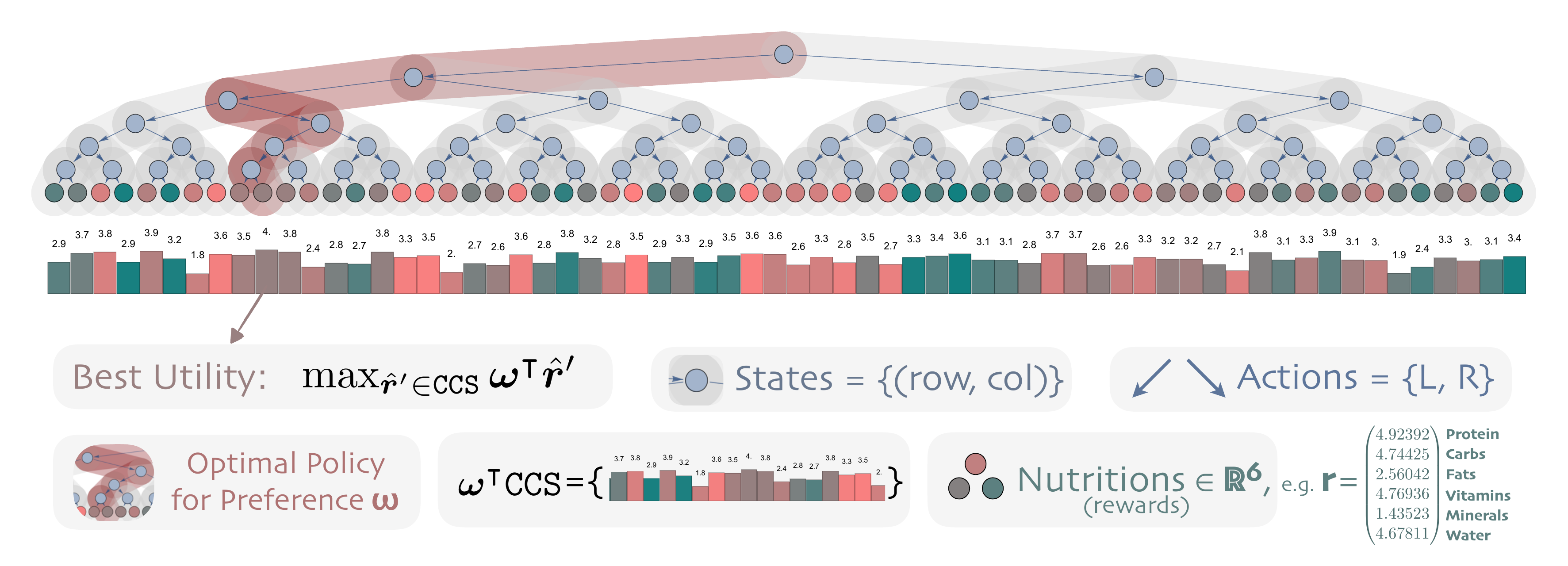}
    \caption{Fruit Tree Navigation (FTN):  An agent travels from the root node to one of the leaf node to pick a fruit according to a post-assigned preference $\bm \omega$ on the components of nutrition, treated as different objectives. The observation of an agent is its current coordinates $(\mathtt{row}, \mathtt{col})$, and its valid actions are moving to the left or the right subtree.}
    \label{fig:ftn-env}
\end{figure}

\paragraph{Task-Oriented Dialog Policy Learning (Dialog)}
Our third experimental domain is a modified task-oriented dialog system in the restaurant reservation domain based on PyDial~\cite{UltesRSVKCBMWGY17}, where an agenda-base user simulator~\cite{SchatzmannY09} with an error model to simulate the recognition and understanding errors arisen in the real system due to in the speech noise and ambiguity. A. We consider the task success rate and the dialog brevity (measured by number of turns) as two competing objectives of this domain. 

Finding a good trade-off between multiple potentially competing objectives is usually domain-specific and not straightforward. 
For example, in the case when the objectives are brevity and success, if the relative importance weight for success is too high, the resulting policy is insensitive to potentially annoying actions such as {\tt repeat()} provided that the dialogue is eventually successful. In this case, the obtained optimal policy cannot fit all users' preference, and sometimes is out of our expectation. Adaptation to user preferences and balancing these objectives is rarely considered. 

In the standard single-objective reinforcement learning formulation, the goal of the policy model is to interact with a human user by choosing actions in each turn to maximize future rewards. We define the dialogue state shared by dialogue state tracker in the $t$-th turn as state $s_t$ . The action taken by policy model under current policy $\pi_{\theta}$ with parameters $\theta$ in the $t$-th turn as $a_t$ , and $a_t\sim\pi(\cdot|s_t)$. The stochastic transition kernel is unknown but determined by human users or user simulators. In an ideal dialogue environment, once the policy model emits an action $a_t$ , the human user will give an explicit feedback, like a normal response or a feedback of whether the dialogue is successful, which will be converted to a reward signal $r_t$ delivering to the policy model immediately, and then the policy model will transit to next state $s_t$ . The reward signal is an average of values of two objectives, brevity and success, e.g., $r_t = 0.5 r_t^{\tt turn} + 0.5 r_t^{\tt succ}$. Typically, $r^{\mathtt{turn}}_t$ is fixed for each turn as a negative constant $R^{\mathtt{turn}}$, while $r^{succ}_t$ equals a positive constant $R^{\mathtt{succ}}$ only when the dialogue terminates and receives a successful user feedback otherwise it equals zero.

To address this problem, we transform the dialogue learning process into a MORL scenario with vectorized (2-D) rewards:
$
	\bm r_t = \left[\begin{matrix}
        r_t^{\mathtt{turn}} & r_t^{\mathtt{succ}}
    \end{matrix}\right]^{\intercal}
$,
where $r_t^{\mathtt{turn}}$ is a turn penalty for the brevity objective and $r_t^{\mathtt{succ}}$ is a reward provided on successful completion of the task. In the learning phase, the linear preference $\bm \omega$ over these two objectives are unknown, while the computational resources are abundant. The task-oriented dialogue system needs to learn all the possible optimal policies with sampled $\bm \omega$'s (achieved by user simulators or collected interactions with real users). While in the adaptation phase, learning is unaffordable because of the limitation of resources. The task-oriented dialogue system needs to respond the user with a specified user preference. User's utility increase is aligned with the system's utility increase $\bm \omega^{\intercal}\bm r_t$. Paper~\cite{UltesBCMRSWGY17} proposes a structured method, which is equivalent to the scalarized baseline without hindsight experience replay, for finding the optimal weights for a multi-objective reward function.

As for more experimental details, our dialog domain is a restaurant reservation hotline which provides information about restaurants in Cambridge. There are 3 search constraints, 9 informational items that the user can request, and 110 database entities. The reward $R^{\tt turn}$ is $-1$ for each turn, and $R^{\tt succ} = 20$. The maximal length of dialogue is 25. We apply our envelope deep MORL algorithm to this dialogue policy learning task, and compare to traditional single-objective methods and other baselines. All the single-objective and multi-objective reinforcement learning are trained for 3,000 sessions with 15\% simulated speech recognition and understanding error rate.

\paragraph{Multi-Objective SuperMario Game (SuperMario)} Our final environment is a version of the popular video game Super Mario Bros. We modify the open-source environment from OpenAI gym~\cite{gym-super-mario-bros} to provide vectorized rewards encoding five different objectives: {\tt x-pos}: value corresponding to the difference in Mario's horizontal position between current and last time point,  {\tt time}: a small negative time penalty, {\tt deaths}: a large negative penalty given each time Mario dies\footnote{Mario has up to three lives in one episode.},  {\tt coin}:  rewards for collecting coins, and {\tt enemy}: rewards for eliminating an enemy. The state is a stack of four continuous frames of game images rendered by the simulator, and there are seven valid actions each step: \{`NOOP',`right',`right+A',`right+B',`right+A+B', `A', `left'\}, where the button `A' is used to jump and the button `B' is used to run. We restrict the Mario to only play the stage I.

We use an A3C~\cite{MnihBMGLHSK16} variant of our envelope MORL algorithm. During the learning phase, the agent does not know the underlying preference, and hence needs to learn a multi-objective policy within 32k training episodes. During the adaptation phase, we test our agents under 500 uniformly random preferences and test the its preference elicitation ability (as described in Section~\ref{s:inference}) within 100 episodes to uncover the underlying preference that maximizes utility.

\subsection{Implementation Details}

Our multi-objective Q-network can be replaced with any model similar to that in single-objective off-policy RL algorithms like DDPG~\cite{LillicrapHPHETS15}, NAF~\cite{GuLSL16} or SDQN~\cite{MetzIJD17}. In the experiment, we use a variate of deep reinforcement learning techniques including double Q-learning~\cite{HasseltGS16} with a target network and prioritized experience replay~\cite{SchaulQAS15}, which stabilize and speed up our algorithms. 

\vspace{-.5em}
\paragraph*{Architectures of the Multi-objective Q-Network} We implement the Multi-objective Q-networks (MQNs) by 4 fully connected hidden layers with  $\{16, 32, 64, 32\} \times (\mathtt{dim}(S)+m)$ hidden unites respectively. The multi-objective Q-networks are similar to Deep Q-Networks (DQNs)~\cite{MnihKSRVBGRFOPB15}, but differs on inputs. An input of the  multi-objective Q-network is a concatenation of state representation and parameters of a linear preference function. The output layer of the scalarized MORL algorithm is of size $|\mathcal A|$, and that of envelope version is of size $m\times |\mathcal A|$. Here $\mathtt{dim}(\mathcal S)$ is the dimensionality of the state space, $|\mathcal A|$ is the cardinality of the action set, and $m$ is the number of objectives.

\vspace{-.5em}
\paragraph*{Multi-Objective A3C}
We use the multi-objective A3C (MoA3C) algorithms for Mario experiment. The skeleton of the envelope MoA3C algorithm is provided in Algorithm \ref{algo:enve-a3c}. Im MoA3C Both critic and actor networks contain three shared convolutional layers for feature extraction from raw images input. The extracted features are then concatenated with preferences, and fed into two-layer fully connected networks for output. For the scalarized version MoA3C, the output of the critic network is just one-dimensional utility prediction, whereas the output of the envelope version critic network is $m$-dimensional returns prediction. Both scalarized and envelope versions have the same actor network architecture to output the probability distribution over the action space. We train them with 16 workers in parallel with different sampled preferences, and it take around 10 hours for the envelope version MoA3C to converge to a good level of performance.

\vspace{-.5em}
\paragraph*{Training with Prioritized Double Q-Learning} When training the with our MORL algorithm on DST and FTN tasks, we employ techniques of {\em prioritized experience reply}~\cite{SchaulQAS15} and {\em double Q-learning}~\cite{HasseltGS16} to speed up the training process and to yield more accurate value estimates. Double Q-Learning introduces a target network $Q_{\textrm{target}}$ to replace the estimate of $\mathcal T Q(s, a, \bm \omega)$, $y_t(\bm \omega) = \bm\omega^{\intercal} \bm r + \gamma Q_{\textrm{target}}(s',a,\bm \omega)$ with $y_t^{\textrm{double}}(\bm \omega) = \bm\omega^{\intercal} \bm r + \gamma  Q_{\textrm{target}}(s', \argmax_a Q(s',a,\bm \omega), \bm \omega)$ for scalarized version of algorithm, and similarly we replace $y_t(\bm \omega) = \bm\omega^{\intercal} \bm r + \gamma \bm \omega^{\intercal} \bm Q_{\textrm{target}}(s',a,\bm \omega)$ with $y_t^{\textrm{double}}(\bm \omega) = \bm\omega^{\intercal} \bm r + \gamma \bm \omega^{\intercal} \bm Q_{\textrm{target}}(s', \argmax_a \bm \omega^{\intercal} \bm Q(s',a,\bm \omega), \bm \omega)$ for envelope version of MORL algorithm.  We update the target network by coping from Q-network every $100$ steps. 
The priority of sampling transition $\tau_i=(s,a, \bm r, s')$ is $p_i^{\textrm{scalarized}}=|y^{\textrm{double}}(\bm \omega) - Q(s, a, \bm \omega)|$ for the scalarized version of MORL algorithm, and similarly $p_i^{\textrm{envelope}}=|y^{\textrm{double}}(\bm \omega) - \bm \omega^{\intercal} \bm Q(s, a, \bm \omega)|$ for the envelope version of algorithm, where $\bm \omega$ is sampled from the distribution $\mathcal D_{\omega}$. When updating the network, a trajectory is sampled by $\tau_{i}\sim P(i) = p_i/\sum_{i}p_i$. The replay memory size is 4000 and the batch size is 32. For the deep tree navigation task, we train each model for total 5000 episodes, and update it by Adam optimizer every step after at least a batch experiences are stored in reply buffer, with a learning rate $\mathtt{lr}=0.001$.

\vspace{-.5em}
\paragraph*{Training Details for Dialogue Policy Learning}
All the single-objective and multi-objective reinforcement learning are trained for 3,000 sessions. We evaluate learned policies on 5,000 sessions with randomly assigned user preferences. The preference distribution $\mathcal D_\omega$ is same as the one we used in previous deep sea treasure experiment (see Section 4.3.1), which is a nearly uniform distribution. For the single-objective reinforcement learning algorithms, we set three groups of $\bm \omega$ as $\{(0.5, 0.5),$ $(0.2,0.8),$ $ (0.8,0.2)\}$. For the envelope deep MORL algorithms, the homotopy path is a monotonically increasing track where $\lambda$ increases from $0.0$ to $1.0$ exponentially. The number of sampled preferences $N_{\omega}$ is 32 for both scalarized and envelope deep MORL algorithms. The exploration policy used for training these reinforcement learning algorithms is $\epsilon$-greedy, where $\epsilon = 0.5$ initially and then decays to zero linearly during the training process. For all the single-objective and multi-objective algorithms, we employ the same deep Q-network architecture, which comprises 3 fully connected hidden layers with $\{16, 32, 32\} \times (\mathtt{dim}(S)+m)$ hidden units. The minibatch size is 64 for all. An Adam optimizer is used for updating the parameters of all these algorithms with an initial learning rate $\mathtt{rl} = 0.001$.

\vspace{-.5em}
\paragraph*{Computing Infrastructure} We ran the synthetic experiments and the dialog experiments on a workstation with one GeForce GTX TITAN X GPU, 12 Intel(R) Core(TM) i7-5820K CPUs @ 3.30GHz, and 32G memory and ran the SuperMario experiments on a cluster with twenty 2080 RTX GPUs, 40 CPUs and 200GB memory.
\begin{algorithm}[t]
    \caption{Envelope Multi-Objective A3C (EMoA3C) Algorithm - Pseudocode for each actor-learner thread}
    \label{algo:enve-a3c}
    \KwIn{
    	\vspace{-0.5em}
        \begin{itemize}
        	\setlength{\itemsep}{0pt}
            \item a preference sampling distribution $\mathcal D_{\omega}$;
            \item minibatch sizes for transitions $N_{\tau}$ and for preference $N_{\omega}$;
            \item a multi-objective critic-network $\bm V$ parameterized by $\theta_{v}$;
            \item an actor-network $\pi$ parameterized by $\theta_{\pi}$;
            \item a balance weight $\lambda$ for critic losses $L^{\tt A}$ and $L^{\tt B}$.
        \end{itemize}
    }
    
    Initialize replay buffer $\mathcal D_{\tau}$.\\
    \For{episode = $1, \dots, M$}{
        Synchronize thread-specific parameters $\theta_{v}' = \theta_{v}$  and $\theta_{\pi}' = \theta_{\pi}$.\\
        Sample a linear preference $\bm \omega \sim \mathcal D_{\omega}$.\\
        \For{t = $0, \dots, N-1$}{
        	Observe state $s_t$.\\
            Sample an action $a_t$ using according to policy $\pi(a_t | s_t, \bm \omega; \theta_\pi')$.\\
            Receive a multi-objective reward $\bm r_t$ and observe new state $s_{t+1}$.\\
            Store transition $(s_t, a_t, \bm r_t, s_{t+1})$ in $\mathcal D_{\tau}$.\\
            \If{update}{
                Sample random minibatch of transitions $(s_j, a_j, \bm r_j, s_{j+1})$ from $\mathcal D_{\tau}$.\\
                Sample $N_{\omega}$ preferences $W = \{\bm \omega_{1}, \bm \omega_{2}, \dots, \bm \omega_{N_{\omega}}\}\sim \mathcal D_{\omega}$.\\
                Compute $(\mathcal T \bm V)_{ij} = \begin{cases}
                    \bm r_j. & \textrm {for terminal~} s_{j+1};\\
                    \bm r_j + \gamma \arg_{V}\max_{a\in\mathcal A, \bm \omega'\in W} \bm \omega^{\intercal}_i\bm V(s_{j+1}, \bm \omega'; \theta), & \textrm {for non-terminal~} s_{j+1}.
                \end{cases}$\\
                Calculate $d\theta_v$ according to similar equations \ref{eq:enve-lossA} and \ref{eq:enve-lossB} w.r.t. $\theta_v'$:
                \begin{equation*}
                   d\theta_v =  (1-\lambda) \cdot \nabla_{\theta_v'}L^{\mathtt A}(\theta_v') + \lambda \cdot \nabla_{\theta_v'}L^{\mathtt B}(\theta_v').    
                \end{equation*}\\
                Calculate $d\theta_\pi$ using the advantage w.r.t. $\theta_v'$ and $\theta_\pi'$ 
                \begin{equation*}
                   d\theta_\pi =  \frac{1}{N_{\omega} N_{\tau}}\sum_{i,j}\nabla_{\theta_\pi'} \log \pi(a_j|s_j, \bm \omega_i;\theta_\pi')\left[\bm \omega_{i}^{\intercal}((\mathcal T \bm V)_{ij} - \bm V(s_j, \bm \omega_i;\theta_v'))\right].    
                \end{equation*}\\
                Perform asynchronous update of $\theta_v$ using $d\theta_v$ and of $\theta_\pi$ using $d\theta_\pi$.\\
            }
        }
    }
\end{algorithm}
\section{Additional Experimental Results}
\label{app:sec:res}

\subsection{Evaluation Metrics}

We design two metrics to evaluate the empirical performance of our algorithms on test tasks. Slightly different from the main article, we introduce adaptation quality (AQ) here other than adaptation error to adjust the value range and get a score in $(0,1]$.

\begin{figure}[t]
    \includegraphics[width=0.9\columnwidth]{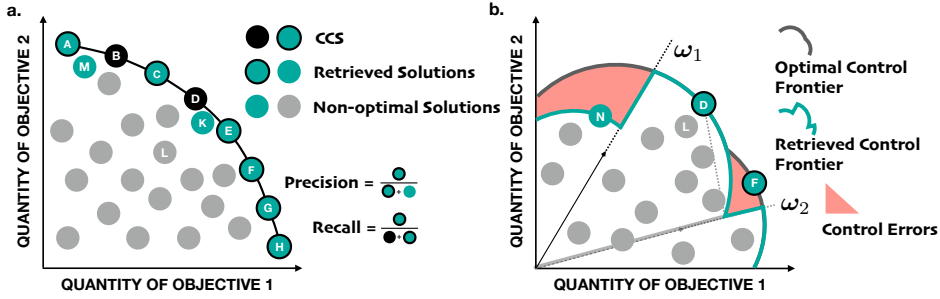}
%    \vspace{-1.5em}
    \caption{Quantitative evaluation metrics for multi-objective reinforcement learning. (a.) Coverage ratio measures an agent's ability to find all the potential optimal solutions in the convex coverage set of Pareto frontier.  (b.) Adaptation quality measures an agent's ability of policy adaptation to real-time specified preferences.}
    \label{fig:eval-metric-2} 
  	\vspace{-0.2em}
\end{figure}    

\textbf{Coverage Ratio (CR).}
\label{app:sec:eval-metic:cr}
The first metric is \textit{coverage ratio} (CR), which evaluates the agent's ability to recover optimal solutions in the convex coverage set ($\mathtt{CCS}$). If $\mathcal F \subseteq \mathbb R^{m}$ is the set of solutions found by the agent (via sampled trajectories),
% and $\mathcal F_{Q} \subseteq \mathbb R^{m}$ is the set of agent's guessed solutions (via predictions from Q-network), 
we define $\mathcal F \cap_{\epsilon}\mathtt{CCS}:=\{x\in \mathcal F \mid \exists y \in \mathtt{CCS} \textrm{~s.t.~} \|x-y\|_{1}/\|y\|_1 \leq \epsilon \}$ as the intersection between these sets with a tolerance of $\epsilon$. 
The CR is then defined as:
\begin{equation}
    \mathtt{CR_{F1}}(\mathcal F) = \mathtt{2} \cdot \frac{ \mathtt{precision\cdot recall}}{\mathtt{precision} + \mathtt{recall}},
\end{equation}
where the $\mathtt{precision} = |\mathcal F \cap_{\epsilon} \mathtt{CCS}|/|\mathtt{\mathcal F}|$, indicating the fraction of optimal solutions among the retrieved solutions, and the $\mathtt{recall} = |\mathcal F \cap_{\epsilon}\mathtt{CCS}|/|\mathtt{CCS}|$, indicating the fraction of optimal instances that have been retrieved over the total amount of optimal solutions (see Figure~\ref{fig:eval-metric}(a)).
The F1 score is their harmonic mean. 
In our evaluation of both synthetic tasks DST and FTN, we set $\epsilon = 0.00$ for $\mathcal F_{\Pi_{\mathcal L}}$ (executive frontier) and $\epsilon = 0.20$ for $\mathcal F_{Q}$ (frontier predicted by Q-function). 
Figure \ref{fig:eval-metric-2} (a.) illustrates an example of the computation of coverage ratio.

\textbf{Adaptation Quality (AQ).}
\label{app:sec:eval-metic:aq}
Our second metric compares the retrieved control frontier with the optimal one, when an agent is provided with a specific preference $\bm\omega$ during the adaptation phase. 
The adaptation quality is defined by
\begin{equation}
    \mathtt{AQ}(\mathcal C) = \frac{1}{1+ \alpha \cdot \mathtt{err}_{\mathcal D_{\omega}}},
\end{equation}
where the $\mathtt{err}_{\mathcal D_{\omega}} = \mathbb E_{\bm \omega \sim \mathcal D_{\omega}} [|\mathcal C(\bm \omega) - \mathcal C_{\textrm{opt}}(\bm \omega)| / \mathcal C_{\textrm{opt}}(\bm \omega)]$ is the expected relative error between optimal control frontier $\mathcal C_{\textrm{opt}}: \Omega \rightarrow \mathbb R$ with $\bm \omega \mapsto \max_{\bm{\hat r}\in \mathtt{CCS}} \bm \omega^{\intercal} \bm{\hat r}$ and the agent's control frontier $\mathcal C_{\pi_{\bm \omega}} = \bm \omega^{\intercal} \bm{\hat r}_{\pi_{\bm \omega}}$ , and $\alpha$ is a scaling coefficient to amplify the discrepancy. 

Similarly, $\mathcal C_{Q}$ is the control frontier guessed by an agent (via predictions from Q-network) and we can compute the predictive ${\tt AQ}(\mathcal C_{Q})$ to evaluate the quality of multi-objective Q-network on the value prediction accuracy.

In all experiment domains, we use Gaussian distributions which are restricted to be positive part and $\ell1$-normalized as our $\mathcal D_{\omega}$. We set $\alpha = 0.01$ for the DST task (because the penalty range is large), and $\alpha = 10.0$ for the FTN task (because the value differences are small). 
Figure \ref{fig:eval-metric-2} (b.) shows examples of optimal control frontier, retrieved control frontier, and the control discrepancy.
Overall, CR provides an indication of agent's ability to learn the space of optimal policies in the learning phase, while AE tests its ability to adapt to new scenarios.

\subsection{Deep Sea Treasure (DST)}

We show more experimental results on DST tasks in this section. We train all agent for 2000 episodes. After training in the learning phase, our envelope MORL algorithms find all the potential optimal solutions and their corresponding policies. 

Figure~\ref{fig:syn-dst-frontier} presents the real CCS and the retrieved solutions of a MORL algorithm. The scalarized and envelope algorithm  can find all the whole CCS. Figure~\ref{fig:syn-dst-frontier} (b.) illustrates the real control frontier (the blue curve), retrieved control frontier (the green curve), and the predicted control frontier (the orange line). The retrieved control frontier is almost overlapped with the real control frontier, which indicates that the alignment between preferences and optimal policies is perfectly well. The agent can respond any given preference with the policy resulting in best utility.

\begin{minipage}{\textwidth}
  \begin{minipage}[b]{0.48\textwidth}
    \centering
    \includegraphics[width=\textwidth]{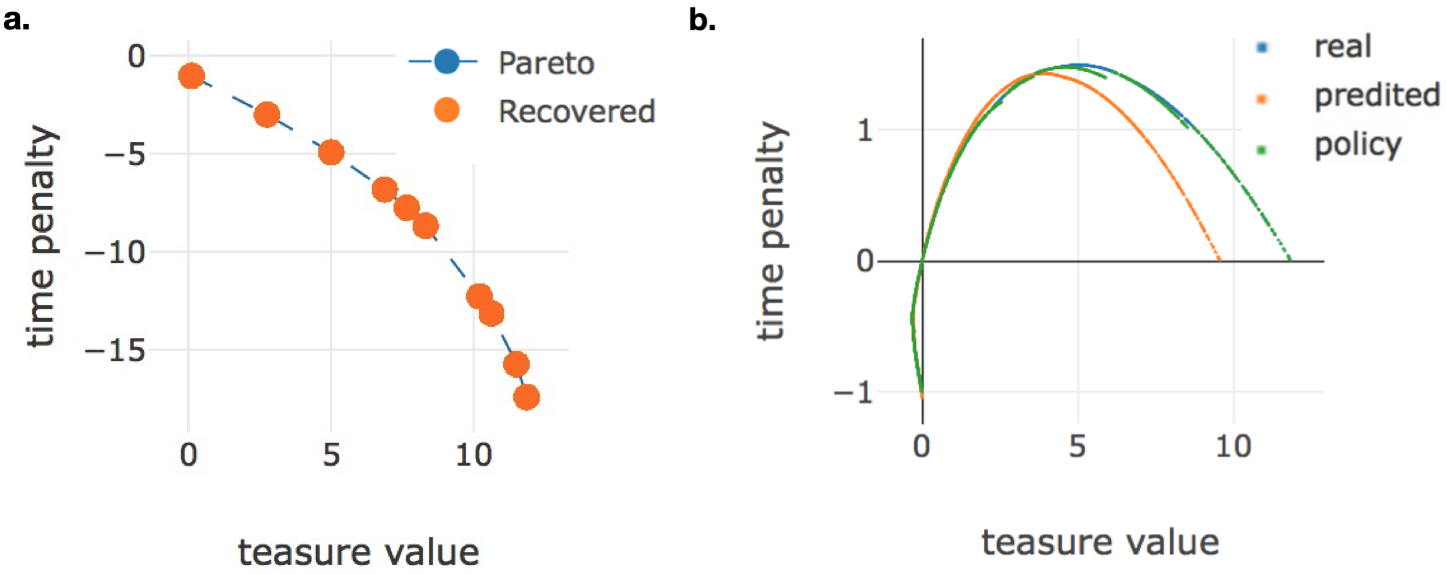}
    \captionof{figure}{The solutions and control frontier found by MORL algorithm in the deep sea treasure task. (a.) The real CCS and the retrieved solutions. (b.) The real control frontier, predicted control frontier, and retrieved control frontier in the adaptation phase.}
    \label{fig:syn-dst-frontier}
  \end{minipage}
  \hfill
  \begin{minipage}[b]{0.50\textwidth}
    \centering
    \resizebox{\columnwidth}{!}{
        \begin{tabular}{cccc}
    	\toprule
		\multirow{2}{3em}{Method} &  \multicolumn{3}{c}{DST} \\
		\cmidrule(lr){2-4}
		{} & {\tt CR F1} & {\tt Exe-AQ ($\alpha = 0.01$)} & {\tt Pred-AQ ($\alpha = 0.01$)}\\
		\midrule
		MOFQI &  
		0.639 $\pm$ 0.421 & 0.417 $\pm$ 0.134 &
		0.226 $\pm$ 0.138
        \\
		CN+OLS&
		0.751 $\pm$ 0.163 & 0.743 $\pm$ 0.008 &
		0.177 $\pm$ 0.089
        \\
		Scalarized &
		0.989 $\pm$ 0.024 & 0.998 $\pm$ 0.001 &
		{\bf 0.950 $\pm$ 0.034}
        \\
		Envelope &
		{\bf 0.994 $\pm$ 0.001} & {\bf 0.998 $\pm$ 0.000} &
		0.850 $\pm$ 0.045
        \\
		\bottomrule
    \end{tabular}
    }
      \captionof{table}{Performance comparison of different MORL algorithm in learning and adaptation phases on the DST environment. The {\tt Exe-AQ} is the AQ measured on the real trajectories in the adaptation phase, and the {\tt Pred-AQ} is the AQ values measured on the predictions of Q functions.}
      \label{tab:dst}
    \end{minipage}
  \end{minipage}
  
Table \ref{tab:dst} provides the coverage ratio (CR) and adaptation quality (AQ) comparisons of different MORL algorithms. We trained all the algorithm in 2000 episodes and test for another 2000 episodes. Each data point in the table is an average of 5 train and test trails. For the CN+OLS baseline, we allow it to iterate for 25 corner weights. The envelope algorithm achieves best CR and execution AQ, and the scalarsized algorithm achieves best predictive adaption quality. Note that traditional evaluations rarely test the algorithm's ability in learning phase and adaptation phase separately. Thus our setting is more challenging.

The classical deep sea treasure task shows the effectiveness of our deep multi-objective reinforcement learning algorithms, while it is relatively easy for the agent to find all the good policies. It only contains 10 potentially optimal solutions in the real CCS, therefore scalarized algorithm can efficiently solve this problem.

\subsection{Fruit Tree Navigation (FTN)}
\begin{figure}[h]
    \centering
    \includegraphics[width=0.98\columnwidth]{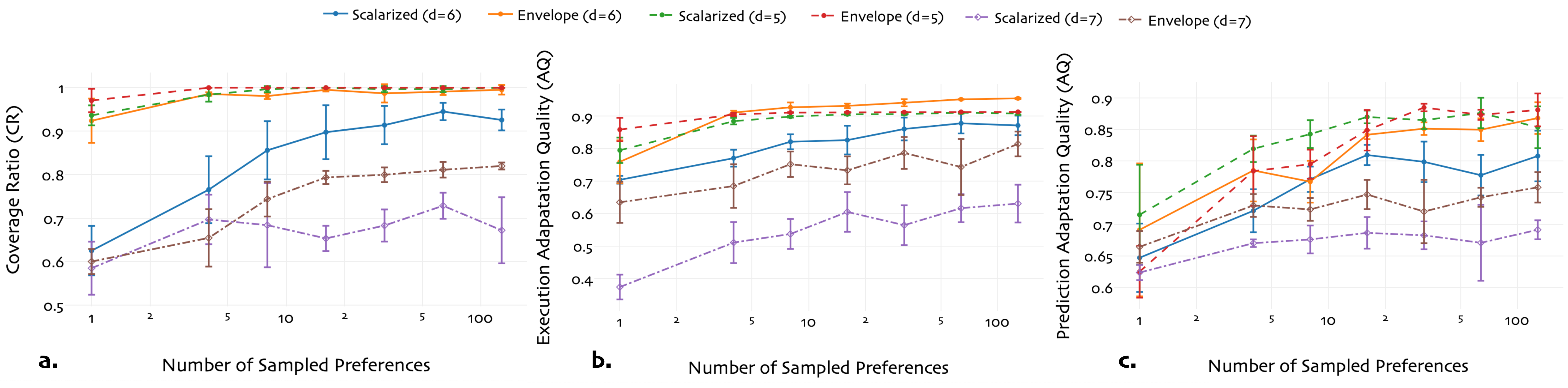}
%    \vspace{-1.5em}
    \caption{Coverage Ratio (CR) and Adaptation Quality (AQ) comparison of the scalarized deep MORL algorithm and the envelope deep MORL algorithm tested on fruit tree navigation tasks of depths d = 5, 6, 7. Trained on 5000 episodes and test on 2000 and 5000 episode to estimate CR and AQ, respectively. Each data point in the figure is an average of 5 trails of training and test.}
    \label{fig:syn-exp-ftn} 
  	\vspace{-0.2em}
\end{figure}    

\paragraph{Sample Efficiency} To compare sample efficiency during the learning phase, we train envelope MORL algorithm and baselines on the FTN task of depth $5,6,7$ for 5000 episodes. We compute coverage ratio (CR) over 2000 episodes and adaptation quality (AQ) over 5000 episodes. Figure \ref{fig:syn-exp-ftn} (blue and orange curves for $d=6$) shows plots for the metrics computed over a varying number of sampled preferences $N_{\omega}$ from $1$ to $128$. When $N_{\omega} = 1$, both algorithms only update under single preference each time, therefore runs fast while it makes wasteful use of interactions. When $N_{\omega} = 128$, both algorithm needs to update for 128 sampled preferences in a batch. In this case the algorithms run slowly, while can make better use of interactions. Each point on the curve is averaged over 5 experiments. We observe that the envelope MORL algorithm consistently has a better CR and AQ scores than the scalarized baseline, with smaller variances. As $N_\omega$ increases, CR and AQ both increase, which shows better use of historical interactions for both algorithms when $N_\omega$ is larger. This reinforces our theoretical analysis that the envelope MORL algorithm has better sample efficiency than the scalarized baseline.

Table \ref{tab:sample-craq-ft6} compares their coverage ratios. Each entry in the table is an average of 5 experiments (train and test). Due to the property of FTN task that every solution is potentially optimal, the CR precision is always 1 for both scalarized and envelope algorithm. While for the recall, which indicates the ability to find unseen optimal solutions in the learning phase, the envelope deep MORL algorithm is better than the scalarized version for all numbers of sampled preferences, and has relatively lower variances. Therefore the same to F1 scores. As $N_{\omega}$ increases, the CR value increases for both scalarized version and envelope version algorithms, which verifies aa better use of historical interactions for both algorithm when $N_{\omega}$ is larger. As here the initial performance for the envelope version algorithm is good enough, it suddenly surpasses 0.98 of CR F1 score when it can sample more than one preference each update.

\begin{table*}[h]
    \centering
    \resizebox{\columnwidth}{!}{
    \begin{tabular}{c|c|c|c|c|c|c}
    & \multicolumn{2}{c}{Execution AQ} & \multicolumn{2}{c|}{Prediction AQ} & \multicolumn{2}{c}{CR F1}\\
    \hline
    $N_{\omega}$ & Scalarized & Envelope& Scalarized & Envelope& Scalarized & Envelope\\
    \hline
    1 &$0.7037 \pm 0.012$ &$0.759 \pm 0.066$ &$0.6474 \pm 0.054$ &$0.6915 \pm 0.105$ & $0.625 \pm 0.057$ & $0.924 \pm 0.051$ \\
4 &$0.7701 \pm 0.026$ & $0.9101 \pm 0.006$ &$0.7216 \pm 0.034$ &$0.7853 \pm 0.049$ & $0.7654 \pm 0.077$ & $0.9856 \pm 0.004$ \\
8 &$0.8205 \pm 0.023$ &$0.9261 \pm 0.015$ &$0.7714 \pm 0.02$ &$0.7675 \pm 0.033$ & $0.856 \pm 0.067$ & $0.9808 \pm 0.007$ \\
16 &$0.8255 \pm 0.044$ &$0.9306 \pm 0.007$ &$0.8097 \pm 0.016$ &$0.8417 \pm 0.007$ & $0.8976 \pm 0.062$ & $0.9952 \pm 0.004$ \\
32 &$0.8597 \pm 0.035$ &$0.9402 \pm 0.011$ &$0.7989 \pm 0.032$ &$0.8513 \pm 0.01$ & $0.914 \pm 0.044$ & $0.987 \pm 0.021$ \\
64 &$0.877 \pm 0.031$ &$0.9506 \pm 0.001$ &$0.7778 \pm 0.032$ &$0.8497 \pm 0.018$ & $0.9452 \pm 0.02$ & $0.9904 \pm 0.007$ \\
128 &$0.8705 \pm 0.03$ &$0.9536 \pm 0.002$ &$0.8081 \pm 0.04$ &$0.868 \pm 0.025$ & $0.9258 \pm 0.024$ & $0.9952 \pm 0.011$ \\
    \hline
    \end{tabular}
    }
    \caption{Sample Efficiency - Coverage Ratio (CR) and Adaptation Quality (AQ) comparison of the scalarized deep MORL algorithm and the envelope deep MORL algorithm tested on fruit tree navigation task, where the tree depth $d=6$. Trained on 5000 episodes.}    
        \vspace{-1em}
    \label{tab:sample-craq-ft6}
\end{table*}

\begin{table*}[h]
    \centering
    \resizebox{\columnwidth}{!}{
    \begin{tabular}{c|c|c|c|c|c|c}
    & \multicolumn{2}{c|}{Execution AQ} & \multicolumn{2}{c|}{Prediction AQ} & \multicolumn{2}{c}{CR F1}\\
    \hline
    $N_{\omega}$ & Scalarized & Envelope& Scalarized & Envelope& Scalarized & Envelope\\
    \hline
$1$  &$0.7943 \pm 0.039$ &$0.8578 \pm 0.036$ &$0.7153 \pm 0.079$ &$0.6254 \pm 0.041$ &$0.9364 \pm 0.023$ &$0.9706 \pm 0.027$ \\
$4$  &$0.8836 \pm 0.01$ &$0.9041 \pm 0.005$ &$0.8195 \pm 0.021$ &$0.7843 \pm 0.055$ &$0.984 \pm 0.016$ &$1 \pm 0$ \\
$8$  &$0.8975 \pm 0.003$ &$0.9099 \pm 0.001$ &$0.8427 \pm 0.022$ &$0.7952 \pm 0.023$ &$0.9968 \pm 0.007$ &$1 \pm 0$ \\
$16$ &$0.9047 \pm 0.003$ &$0.9109 \pm 0.001$ &$0.8698 \pm 0.01$ &$0.8488 \pm 0.032$ &$1 \pm 0$ &$1 \pm 0$ \\
$32$ &$0.9054 \pm 0.003$ &$0.9113 \pm 0.001$ &$0.8647 \pm 0.014$ &$0.8847 \pm 0.006$ &$0.9968 \pm 0.007$ &$1 \pm 0$ \\
$64$ &$0.9096 \pm 0.003$ &$0.9119 \pm 0$ &$0.8761 \pm 0.024$ &$0.8731 \pm 0.008$ &$0.9968 \pm 0.007$ &$1 \pm 0$ \\
$128$&$0.9071 \pm 0.004$ &$0.9121 \pm 0$ &$0.8535 \pm 0.033$ &$0.8809 \pm 0.026$ &$1 \pm 0$ &$1 \pm 0$ \\
    \hline
    \end{tabular}
    }
    \caption{Sample Efficiency when $d=5$ - Coverage Ratio (CR) and Adaptation Quality (AQ) comparison of the scalarized and the envelope deep MORL algorithms tested on fruit tree navigation task, where the tree depth $d=5$. Trained on 5000 episodes.}    
    \label{tab:sample-craq-ft5}
        \vspace{-1em}
\end{table*}
\begin{table*}[t!]
    \centering
    \resizebox{\columnwidth}{!}{
    \begin{tabular}{c|c|c|c|c|c|c}
    & \multicolumn{2}{c|}{Execution AQ} & \multicolumn{2}{c|}{Prediction AQ} & \multicolumn{2}{c}{CR F1}\\
    \hline
    $N_{\omega}$ & Scalarized & Envelope& Scalarized & Envelope& Scalarized & Envelope\\
    \hline
$1$  &$0.3748 \pm 0.038$ &$0.6348 \pm 0.063$ &$0.624 \pm 0.012$ &$0.6646 \pm 0.025$ &$0.5847 \pm 0.061$ &$0.60 \pm 0.029$ \\
$4$  &$0.5112 \pm 0.063$ &$0.6846 \pm 0.067$ &$0.6704 \pm 0.006$ &$0.730 \pm 0.018$ &$0.6969 \pm 0.057$ &$0.6544 \pm 0.066$ \\
$8$  &$0.5376 \pm 0.046$ &$0.7516 \pm 0.039$ &$0.6763 \pm 0.022$ &$0.7237 \pm 0.018$ &$0.6837 \pm 0.097$ &$0.7437 \pm 0.04$ \\
$16$ &$0.6052 \pm 0.061$ &$0.7328 \pm 0.043$ &$0.6867 \pm 0.025$ &$0.7473 \pm 0.023$ &$0.6532 \pm 0.029$ &$0.7936 \pm 0.015$ \\
$32$ &$0.5646 \pm 0.061$ &$0.7862 \pm 0.049$ &$0.6828 \pm 0.022$ &$0.7204 \pm 0.05$ &$0.6829 \pm 0.037$ &$0.7997 \pm 0.017$ \\
$64$ &$0.6168 \pm 0.043$ &$0.743 \pm 0.086$ &$0.671 \pm 0.06$ &$0.7429 \pm 0.015$ &$0.7284 \pm 0.03$ &$0.8112 \pm 0.018$ \\
$128$&$0.6308 \pm 0.058$ &$0.8138 \pm 0.038$ &$0.6916 \pm 0.015$ &$0.7586 \pm 0.024$ &$0.6719 \pm 0.076$ &$0.8199 \pm 0.008$ \\
    \hline
    \end{tabular}}
    \caption{Sample Efficiency when $d=7$ - Coverage Ratio (CR) and Adaptation Quality (AQ) comparison of the scalarized and the envelope deep MORL algorithms tested on fruit tree navigation task, where the tree depth $d=7$. Trained on 5000 episodes.}    
    \label{tab:sample-craq-ft7}
        \vspace{-1em}
\end{table*}

Table \ref{tab:sample-craq-ft6} also compares MORL algorithms' adaptation quality in the adaptation phase. Each entry in the table is an average of 5 experiments (train and test). For all numbers of sampled preferences, the envelope deep MORL algorithm has better execution AQ than the scalarized algorithm, in spite of better CR of the envelope version, it also indicates the envelope version  algorithm has better alignment between preferences and policies. As $N_{\omega}$ increases, the values of execution AQ and prediction AQ of both algorithms keep increase. Note that even though when $N_{\omega} = 1$ the execution AQ of the scalarized version and the envelope version differs only around $0.5$, when $N_{\omega}$ increase to $4$, the envelope version algorithm better utilizes the sampled preferences to improve the execution AQ to above 0.9. This agrees with our theoretical analysis that our envelope deep MORL algorithm has better sample efficiency than the scalarized version.

We also investigate how the size of optimality frontiers will affect the performance of our algorithms. We train our deep MORL algorithms on two new FTN environments with $d=5$ and $d=7$ respectively. One is smaller than the previous environment, which contains only $32$ optimal solutions, the other is larger than the previous environment, containing $128$ solutions on the CCS. We fix the number of episode for training as $5000$, and test $2000$ episode to obtain coverage ratio, and test $5000$ episode for policy adaptation quality. 

Table \ref{tab:sample-craq-ft5} shows the results of coverage ratio evaluation in the environment of tree depth $d=5$. As it shows, both scalarized and envelope algorithms work well in that environment. the CR F1 scores are very close to 1. The envelope version algorithm is more stable than the scalarized version deep MORL algorithm. Besides, the envelope deep MORL algorithm can predict multi-objective solutions, while the scalarized algorithm cannot. The prediction ability will also be improved as $N_{\omega}$ increases.

\paragraph{Visualizing Frontiers}
We also provide a visualization of the convex coverage set (CCS) and control frontier for our envelope algorithm and the scalarized baseline in Figure~\ref{fig:syn-vis}. 
The left figure shows the real CCS and retrieved CCS of both MORL algorithms using t-SNE~\cite{t-SNE}. We observe that the envelope algorithm (green dots) almost completely covers the entire set of optimal solutions in the real CCS whereas the scalarized algorithm (pink dots) does not. 
The right figure presents the slices of optimal control frontier and the control frontier of our algorithms along the $\mathtt{Minerals}$-$\mathtt{Water}$ plane. The envelope version retrieves almost the entire optimal frontier, while the scalarized version algorithm has larger control discrepancies. The small discrepancy between the real control frontier and the one retrieved by envelope version algorithm at the indentation of the frontier indicates an alignment issue between the preferences and optimal policies.

\begin{figure}[t]
    \centering
    \includegraphics[width=0.98\columnwidth]{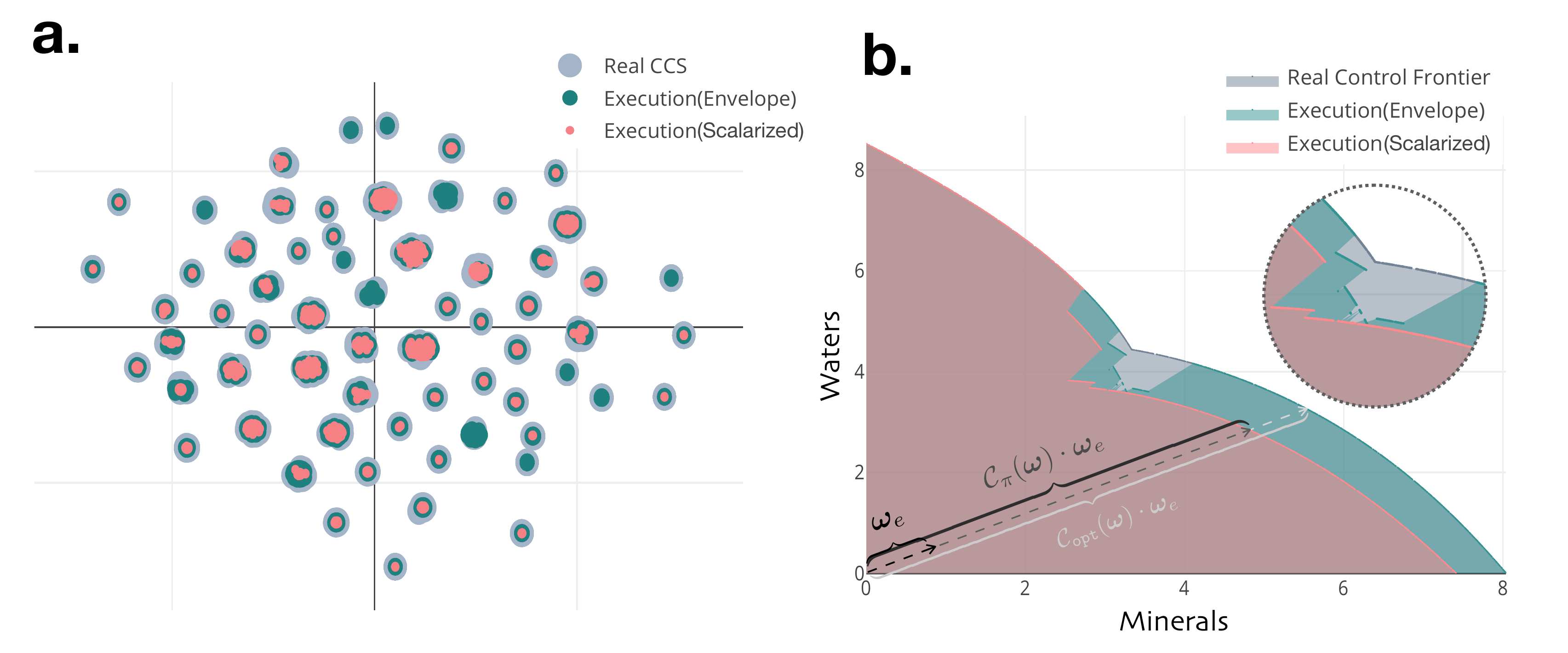}
%    \vspace{-1.5em}    
    \caption{Comparison of CCS and control frontiers of deep MORL algorithms. The left figure (a) is visualizing the real CCS and retrieved CCS of scalarized and envelope MORL algorithms using t-SNE. The right figure (b) presents the slices of optimal control frontier and the control frontier of scalarized and envelope MORL algorithms along the $\mathtt{Mineral}$-$\mathtt{Waters}$ plane.}
    \label{fig:syn-vis}
\end{figure}

\subsection{Task-Oriented Dialog Policy Learning (Dialog)}
\label{supp:sec:dialog}

\begin{table}[h]
    \centering
    \resizebox{\columnwidth}{!}{
    \begin{tabular}{cccccc}
    \toprule
        &Single-(0.5,0.5) & Single-(0.2,0.8) & Single-(0.8,0.2) & Scalarized & Envelope\\
        \hline
    SR & 88.18 $\pm$ 0.90 & 85.30 $\pm$ 0.98 & 87.62 $\pm$ 0.91 & 86.38 $\pm$ 0.95 & {\bf 89.52 $\pm$ 0.85}\\
    \#T & 8.93 $\pm$ 0.13 & 9.40 $\pm$ 0.16 & {\bf 7.42 $\pm$ 0.10} & 8.08 $\pm$ 0.12 & 8.08 $\pm$ 0.12\\
    UT & 2.13 $\pm$ 0.23 & 1.84 $\pm$ 0.23 & 2.53 $\pm$ 0.22 & 2.38 $\pm$ 0.22 & {\bf 2.65 $\pm$ 0.22}\\
    \hline
    AQ  & 0.660 & 0.279 & 0.728 & 0.614 & {\bf 0.814}\\
    \hline
    \end{tabular}}
    \vspace{0.5em}
    \caption{The average success rate (SR), number of turns (\#T), user utility (UT), and adaptation quality (AQ, $\alpha=0.1$) of policies obtained by single-objective RL baselines and two MORL algorithms.}
    \label{tab:dialogue}
\end{table}

\begin{figure}[t!]
    \centering
    \includegraphics[width=0.9\textwidth]{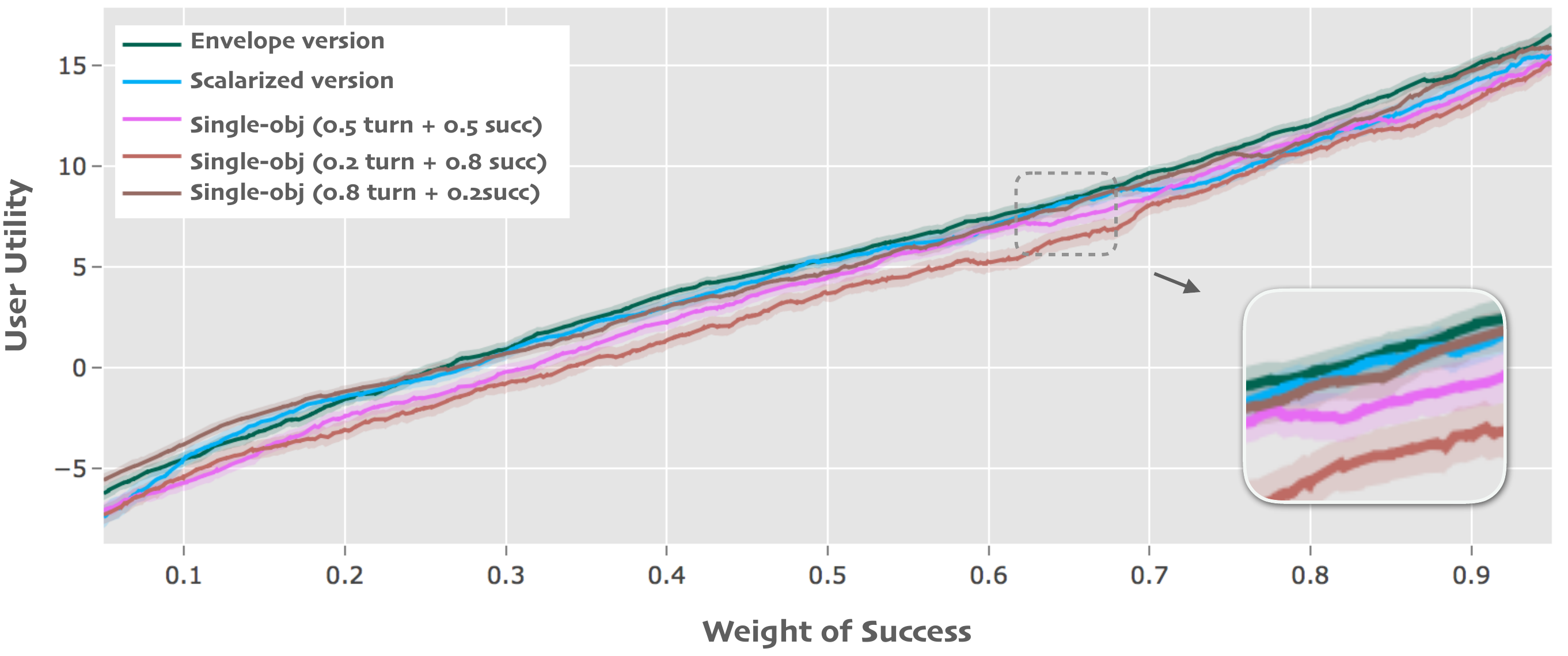}
    \caption{Utility-weight curves for the MORL and single-objective RL dialog policy learning after 3,000 training dialogues. We evaluate policies on 5,000 dialogues with near-uniformly randomly sampled preference. For each curve, each data point is a moving average of the closest 500 dialogues in the interval of around $\pm$ 0.05 weight of success over three trained policies.}
    \label{fig:dialog-utility}
\end{figure}

Table~\ref{tab:dialogue} shows that envelope MORL algorithm achieves best success rate (SR) on average, while the single-objective method with equal weight on both dimensions (0.5) achieves competitive performance. However, on metrics of average utility (UT) and adaptation quality (AQ) (with $\alpha=0.1$, envelope MORL is significantly better than the other methods, including scalarized MORL. This again demonstrates that adaptability of envelope MORL to tasks with new preferences.
The single-objective method with 0.8 success weight has the worst performance on success rate eventually. This is might because lack of turn level reward guidance directly optimize for success is very difficult for the single-objective RL learner. 

Figure~\ref{fig:dialog-utility} illustrates utility-weight curves during a policy adaptation phase with given preferences, where each data point is a moving average of closest 500 dialogues in the interval of around $\pm$ 0.05 weight of success over three trained policies. We find the envelope deep MORL algorithm is almost always better than other methods in terms of utility under certain given preferences, and the scalarized MORL baseline keeps a good level of utility under almost all user preferences. Single-objective RL algorithms are good only when the user's weight of success is close to their fixed preferences while training.

\subsection{Multi-Objective SuperMario Game (SuperMario)}
\label{supp:sec:mario}

\begin{minipage}{\textwidth}
    \centering
    \resizebox{\columnwidth}{!}{
        \begin{tabular}{cccc}
    	\toprule
		\multirow{2}{3em}{Method} &  \multicolumn{3}{c}{Super Mario} \\
		\cmidrule(lr){2-4}
		{} & {\tt Avg.}{\tt UT} ($0.5$ {\tt x-pos} \& $0.5$ {\tt time}) & {\tt Avg.}{\tt UT} ($0.5$ {\tt coin} \& $0.5$ {\tt enemy}) & {\tt Avg.}{\tt UT} (uniform preference)\\
		\midrule
		Scalarized &
		317.1 $\pm$ 123.7 & 76.7 $\pm$ 36.5 &
		301.7 $\pm$ 49.2
        \\
		Envelope &
		{\bf 600.9 $\pm$ 114.9} & {\bf 233.3 $\pm$ 31.2} &
		{\bf 319.7 $\pm$ 34.4}
        \\
		\bottomrule
    \end{tabular}}
      \captionof{table}{Performance comparison of different MORL algorithm in learning and adaptation phases on the SuperMario environment under three different preferences. The first preference emphasizes the fast completion of the task so it is 0.5 on {\tt x-pos} and 0.5 {\tt time}, the second preference emphasizes collecting {\tt coin} and eliminating {\tt enemy}. The third is a uniform preference which has weights 0.2 for all five objectives \{{\tt x-pos}, {\tt time}, {\tt death}, {\tt coin}, {\tt enemy}\}}
      \label{tab:mario}
  \end{minipage}

We compared the average utility of the scalarized baseline and the MORL algorithm with envelope update. The first preference emphasizes the fast completion of the task so it is 0.5 on {\tt x-pos} and 0.5 {\tt time}, the second preference emphasizes collecting {\tt coin} and eliminating {\tt enemy}. The third is a uniform preference which has weights 0.2 for all five objectives \{{\tt x-pos}, {\tt time}, {\tt death}, {\tt coin}, {\tt enemy}\}. The envelope algorithm outperforms the scalarized algorithm under all these preferences.

\begin{figure}[h]
    \centering
    \includegraphics[width=0.9\textwidth]{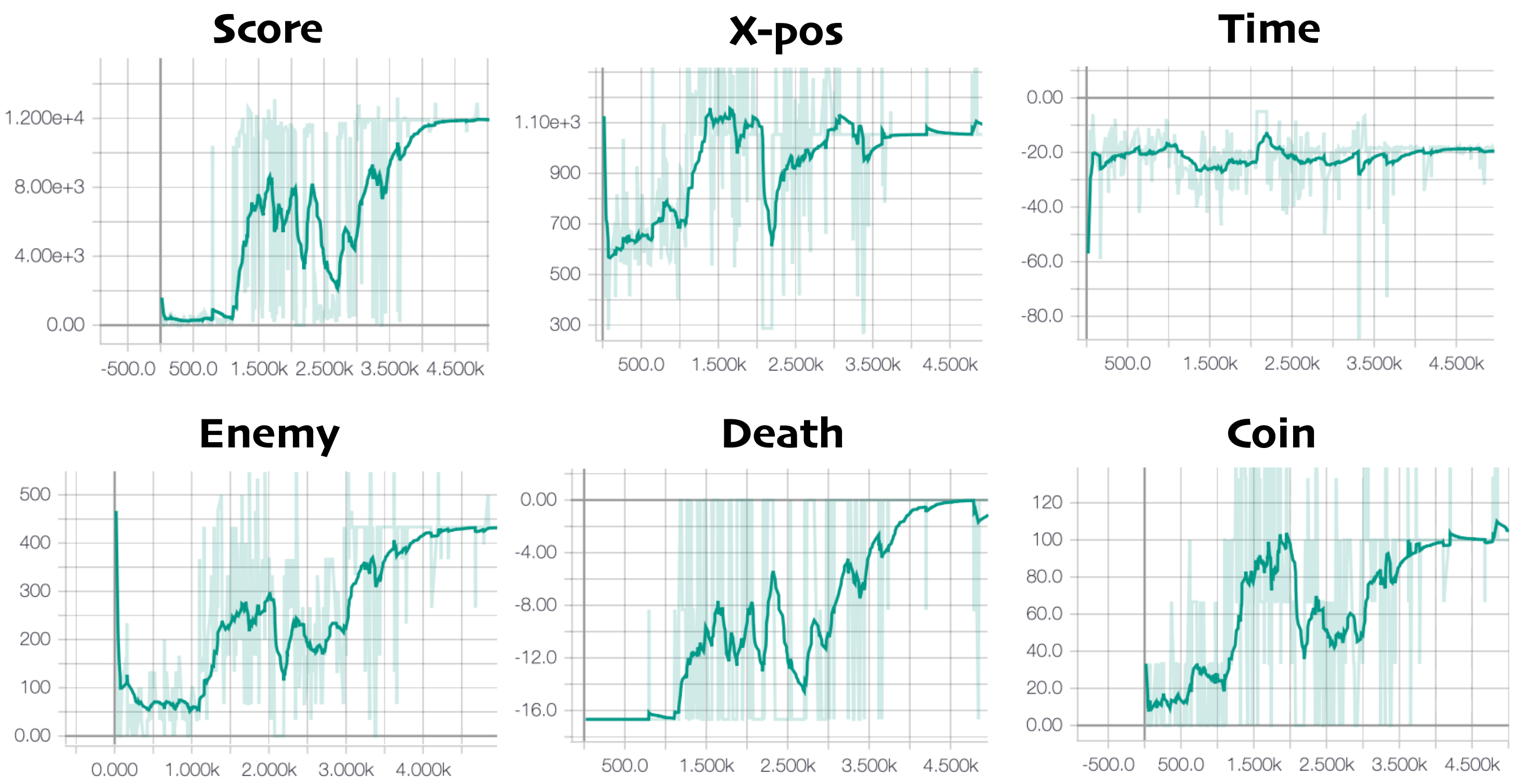}
    \caption{The training curves of the Envelope Multi-Objective A3C (EMoA3C) algorithm.}
    \label{fig:mario-training}
\end{figure}

The training curves of EMoA3C algorithm are shown in  Figure~\ref{fig:mario-training}. To plot Figure~\ref{fig:mario-training}, we use a uniform preference $[0.2~0.2~0.2~0.2~0.2]$ as the probe preference to sample trajectories for evaluation. We observe the algorithm converges within around 5k episodes of training.

\begin{table}[t]
\parbox{.45\linewidth}{
	\centering
	\caption{\small Inferred preferences of the scalarized multi-objective A3C algorithm in different Mario Game variants with 100 episodes.}
	\label{tab:ftn-naive-infer}
	\resizebox{0.48\columnwidth}{!}{
	\begin{tabular}{c|c|c|c|c|c}
	 & {\tt x-pos} & {\tt time} & {\tt life} & {\tt coin} & {\tt enemy} \\ \hline
     {\tt g1}& \cellcolor{pink!23} \underline{0.2315} &\cellcolor{pink!26} 0.2569 &\cellcolor{pink!34}0.335 &\cellcolor{pink!05}0.0522 &\cellcolor{pink!12}0.1243\\ \hline
     {\tt g2}& \cellcolor{pink!16}0.1600 &\cellcolor{pink!14}\underline{0.1408} &\cellcolor{pink!28}0.2822 &\cellcolor{pink!19}0.1955 &\cellcolor{pink!22}0.2216\\ \hline
     {\tt g3}& \cellcolor{pink!18}0.1865 &\cellcolor{pink!18}0.1837 &\cellcolor{pink!08}\underline{0.0874} &\cellcolor{pink!37}0.3759 &\cellcolor{pink!16}0.1665\\ \hline
     {\tt g4}& \cellcolor{pink!12}0.1250 &\cellcolor{pink!30}0.3062 &\cellcolor{pink!35}0.3522 &\cellcolor{pink!12}\underline{0.1240} &\cellcolor{pink!09}0.0926\\ \hline
     {\tt g5}& \cellcolor{pink!23}0.2378 &\cellcolor{pink!16}0.1673 &\cellcolor{pink!18}0.1820 &\cellcolor{pink!34}0.3432 &\cellcolor{pink!06}\underline{0.0698}\\ \hline
	\end{tabular}}
}
~~~~~~
\parbox{.45\linewidth}{
    \centering
	\caption{The difference between the inferred preferences of the envelope and the scalarized multi-objective A3C algorithms.}
	\label{tab:mario-enve-infer-2}
	\resizebox{0.49\columnwidth}{!}{
	\begin{tabular}{c|c|c|c|c|c}
    	& {\tt x-pos} & {\tt time} & {\tt life} & {\tt coin} & {\tt enemy} \\ \hline
      {\tt g1}
      &\cellcolor{pink!30}\underline{+0.2973} &-0.0799 &-0.1850 &-0.0052 &-0.0271\\ \hline
      {\tt g2}
      &\cellcolor{pink!4}+0.0385
      &\cellcolor{pink!8}\underline{+0.0829} &-0.0337
      &-0.0533
      &-0.0348\\ \hline
      {\tt g3}
      &\cellcolor{pink!3}+0.0331
      &-0.0541 &\cellcolor{pink!27}\underline{+0.2667} &-0.1967
      &-0.0490\\ \hline
      {\tt g4}&-0.1039 &-0.0658 &-0.3311 &\cellcolor{pink!57}\underline{+0.5720} &-0.0715\\ \hline
      {\tt g5}&-0.1663&-0.0635 &\cellcolor{pink!2}+0.0249 &\cellcolor{pink!5}+0.0490 &\cellcolor{pink!16}\underline{+0.1555}\\ \hline
    \end{tabular}}
}
\end{table}

Finally we compare the difference between the inferred preferences of the envelope and the scalarized multi-objective A3C algorithm on different variants of SuperMario game, g1 to g5, where only corresponding scalar rewards are available. We only consider maximizing one objective in each game variant, because this orthogonal design helps us characterize the behaviors of agents and compare their inferred preferences. We allow agents to make 100 episodes interactions in each game variance, to determine the preference. Note that only 100 episodes are far from enough for training a single objective reinforcement learning agent, even the model is pre-trained on other tasks. 

As Table~\ref{tab:mario-enve-infer-2} illustrates, the preferences inferred by the envelope MoA3C agent is more concentrate on the diagonal than that of scalarized MoA3C agent, which is more closer to the true underlying preferences.

Even though the EMoA3C agent can do better, our experiments show that the inferred preferences are not exactly the true underlying preferences. It is mainly for two reasons: First, the trade-off frontiers and policy-preference alignment learned by the algorithm is not ideal. There might be some discrepancies between the obtained control frontier and the real optimal control frontier just like what Figure \ref{fig:syn-vis} (b) illustrates. Second, even the agent perfectly learned the frontier and the alignment, close preferences may correspond to the policies with the same expected returns.

We also deploy the trained EMoA3C agents in a Mario Game where only game scores are available. After 100 episodes adaptation, the EMoA3C agent infers that the underlying preference for the achieving higher score mainly focuses on {\tt x-pos} (0.3725) and {\tt time} (0.2307), which is coincident with the strategy human players commonly use – to achieve higher score, especially the stage accomplishment bonus, the first priority is to ensure Mario can move forward towards the flag within the time limit.

% \begin{table}[H]
%     \centering
%     \begin{tabular}{c|c|c|c|c|c}
%     	& {\tt x-pos} & {\tt time} & {\tt life} & {\tt coin} & {\tt enemy} \\ \hline
%       {\tt g1}
%       &\cellcolor{pink!53}\underline{0.5288} &\cellcolor{pink!18}0.1770 &\cellcolor{pink!15}0.1500 &\cellcolor{pink!04}0.0470 &\cellcolor{pink!09}0.0972\\ \hline
%       {\tt g2}
%       &\cellcolor{pink!19}0.1985
%       &\cellcolor{pink!22}\underline{0.2237} &\cellcolor{pink!24}0.2485 
%       &\cellcolor{pink!14}0.1422 
%       &\cellcolor{pink!18}0.1868\\ \hline
%       {\tt g3}
%       &\cellcolor{pink!22}0.2196
%       &\cellcolor{pink!13}0.1296 &\cellcolor{pink!35}\underline{0.3541} &\cellcolor{pink!18}0.1792
%       &\cellcolor{pink!12}0.1175\\ \hline
%       {\tt g4}&\cellcolor{pink!02}0.0211 &\cellcolor{pink!24}0.2404 &\cellcolor{pink!02}0.0211 &\cellcolor{pink!69}\underline{0.6960} &\cellcolor{pink!02}0.0211\\ \hline
%       {\tt g5}&\cellcolor{pink!07}0.0715&\cellcolor{pink!10}0.1038 &\cellcolor{pink!20}0.2069 &\cellcolor{pink!39}0.3922 &\cellcolor{pink!22}\underline{0.2253}\\ \hline
%     \end{tabular}
%     \caption{\small Inferred preferences of the envelope multi-objective A3C algorithm in different Mario Game variants with 100 episodes.}
%     \label{tab:my_label}
% \end{table}
% \input{app-ttest}

\end{document}